%% file: tech-report.tex
\documentclass[english,a4paper]{lipics-v2018}

\usepackage[T1]{fontenc}
\usepackage[utf8]{inputenc}
\usepackage{microtype}

\usepackage{amsmath,amssymb,amsthm}
\usepackage{mathtools}
\usepackage{microtype}
\usepackage{todonotes}
\usepackage{graphicx}
\usepackage{hyperref}

\theoremstyle{plain}
\newtheorem{proposition}[theorem]{Proposition}

\makeatletter
\let\epsilon\varepsilon
\let\phi\varphi
\let\emptyset\varnothing
\let\rho\varrho
\makeatother

\usepackage{tikz}
\usetikzlibrary{positioning,arrows,automata,shapes} 
\tikzset{->,>=stealth',
every loop/.style={looseness=6},
initial text={},
auto,node distance=1cm,
el/.style={font=\scriptsize},
every state/.style={font=\scriptsize,inner
sep=0.1mm,minimum size=0.5cm,outer sep=1pt},
psplit/.style={rectangle,fill=black,minimum size=1mm,inner sep=0mm},
}

\newcommand{\calA}{\mathcal{A}}
\newcommand{\calC}{\mathcal{C}}
\newcommand{\calG}{\mathcal{G}}
\newcommand{\calM}{\mathcal{M}}
\newcommand{\calN}{\mathcal{N}}
\newcommand{\calT}{\mathcal{T}}
\renewcommand{\st}{\mathrel{\mid}}
\newcommand{\defequals}{\overset{\text{def}}{=}}
\newcommand{\defiff}{\overset{\text{def}}{\iff}}

\newcommand{\pow}[1]{\mathcal{P}\left(#1\right)}
\newcommand{\dist}[1]{\mathcal{D}\left(#1\right)}
\newcommand{\runs}[2]{\mathrm{Runs}^{#2}(#1)}
\newcommand{\supp}[1]{\mathrm{supp}\left(#1\right)}
\newcommand{\last}{\mathrm{last}}
\newcommand{\rvprob}[1]{\mathbb{P}\left[#1\right]}
\newcommand{\rvexpect}[1]{\mathbb{E}\left[#1\right]}
\newcommand{\probevent}[3]{\mathbb{P}_{#1}^{#2}\left[#3\right]}
\newcommand{\expect}[3]{\mathbb{E}^{#2}_{#1}\left[#3\right]}
\newcommand{\Inf}{\mathrm{Inf}}

\newcommand{\MP}{\underline{\mathbf{MP}}}
\newcommand{\finMP}{\mathbf{MP}}
\newcommand{\Val}{\mathbf{Val}}
\newcommand{\cVal}{\mathbf{sVal}}
\newcommand{\lcVal}{\mathbf{asVal}}
\newcommand{\parity}{\text{\textsc{Parity}}}
\newcommand{\pmin}{\pi_{\mathrm{min}}}
\newcommand{\autotomdp}[3]{#1_{#2,#3}}
\newcommand{\autotochain}[4]{#1_{#2,#3}^{#4}}
\newcommand{\mec}[1]{\mathrm{MEC}_{#1}}
\newcommand{\eclose}[2]{\mathrel{\sim^{#1}_{#2}}}

\title{Learning-Based Mean-Payoff Optimization
in an Unknown MDP under Omega-Regular Constraints}
\titlerunning{Learning-Based Mean-Payoff Optimization in an Unknown Parity MDP}

\author{Jan K\v{r}et\'insk\'y}{
Technische Universit\"at M\"unchen, Munich, Germany}{jan.kretinsky@in.tum.de}{%
https://orcid.org/0000-0002-8122-2881}{}
\author{Guillermo A. P\'erez}{%
Universit\'e libre de Bruxelles, Brussels, Belgium}{gperezme@ulb.ac.be}{%
https://orcid.org/0000-0002-1200-4952}{%
G. A. P\'erez has been supported by an F.R.S.-FNRS Aspirant fellowship.
}
\author{Jean-Fran\c{c}ois Raskin}{
Universit\'e libre de Bruxelles, Brussels, Belgium}{jraskin@ulb.ac.be}{}{%
J.-F. Raskin is Professeur Francqui de Recherche funded by the Francqui
foundation.
}

\authorrunning{J. K\v{r}et\'insk\'y, G. A. P\'erez, J.-F. Raskin}
\Copyright{Jan K\v{r}et\'insk\'y, Guillermo A. P\'erez, and
Jean-Fran\c{c}ois Raskin}
\subjclass{
\ccsdesc[500]{Theory of computation~Logic and verification},
\ccsdesc[500]{Theory of computation~Reinforcement learning}
}
\keywords{Markov decision processes, Reinforcement learning, Beyond worst case}
\category{}%optional, e.g. invited paper

\relatedversion{}
\supplement{}
\funding{This research was funded in part by the Czech Science Foundation grant
    No.~18-11193S, the German Research Foundation (DFG) project 383882557
    ``Statistical Unbounded Verification'', the ERC Starting grant
    279499 ``inVEST'', the ARC (F\'ed\'eration Wallonie-Bruxelles) project
    ``Non-Zero Sum Game Graphs: Applications to Reactive Synthesis and Beyond'',
    and the EOS (FNRS-FWO) project 30992574 ``Verifying Learning Artificial
    Intelligence Systems''.}
\acknowledgements{}

\EventEditors{John Q. Open and Joan R. Access}
\EventNoEds{2}
\EventLongTitle{42nd Conference on Very Important Topics (CVIT 2016)}
\EventShortTitle{CVIT 2016}
\EventAcronym{CVIT}
\EventYear{2016}
\EventDate{December 24--27, 2016}
\EventLocation{Little Whinging, United Kingdom}
\EventLogo{}
\SeriesVolume{42}
\ArticleNo{23}
\nolinenumbers %uncomment to disable line numbering
\hideLIPIcs  %uncomment to remove references to LIPIcs series (logo, DOI, ...),
\begin{document}

\maketitle

\begin{abstract}
    We formalize the problem of maximizing the mean-payoff value with high
    probability while satisfying a parity objective in a Markov decision process
    (MDP) with unknown probabilistic transition function and unknown reward
    function. Assuming the support of the unknown transition function and a
    lower bound on the minimal transition probability are known in advance, we
    show that in MDPs consisting of a single end component, two combinations of
    guarantees on the parity and mean-payoff objectives can be achieved
    depending on how much memory one is willing to use.
        (i) For all $\epsilon$ and $\gamma$ we can construct an
            online-learning finite-memory strategy 
            that almost-surely satisfies the
            parity objective and which achieves an $\epsilon$-optimal mean
            payoff with probability at least $1 - \gamma$.
        (ii) Alternatively, for all $\epsilon$ and $\gamma$ 
            there exists an online-learning infinite-memory
            strategy that satisfies the parity objective
            surely and which achieves an $\epsilon$-optimal mean payoff with
            probability at least $1 - \gamma$.
    We extend the above results to MDPs consisting of more than one end
    component in a natural way.  Finally, we show that the aforementioned
    guarantees are tight, i.e. there are MDPs for which stronger
    combinations of the guarantees cannot be ensured.    
\end{abstract}

\section{Introduction}
\input{intro.tex}

\section{Preliminaries}
\input{prelim.tex}

\section{Learning for MP: the Unconstrained Case}\label{sec:mp-opt-ec}
\input{mp.tex}

\section{Learning for MP under a Sure Parity Constraint}
\input{sure.tex}

\section{Learning for MP under an Almost-Sure Parity Constraint}
\input{almost.tex}

\section{Conclusion}
\input{conclu.tex}

\clearpage
\bibliographystyle{plainurl}
\bibliography{refs}

\clearpage
\appendix
\input{app.tex}

\clearpage
\setcounter{tocdepth}{5}
\tableofcontents

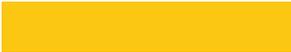
\end{document}

%% file: intro.tex
%intro.tex

\subparagraph{Reactive synthesis and online reinforcement learning.}
Reactive systems are systems that maintain a continuous interaction with the
environment in which they operate. When designing such systems, we usually face
two partially-conflicting objectives. First, to ensure a safe execution, we want
some basic and critical properties to be enforced by the system no matter how
the environment behaves. Second, we want the reactive system to be as efficient
as possible given the actual observed behaviour of the environment in which the
system is executed. As an illustration, let us consider a robot that needs to
explore an unknown environment as efficiently as possible while avoiding any
collision.  While operating at low speed makes it easier to avoid
collisions, it will impair its ability to explore the environment quickly even
if the environment is clear of other objects.  

There has been, in the past, a large research effort to define mathematical
models and algorithms in order to address the two objectives above, but in
isolation only. To synthesize safe control strategies, two-player zero-sum games
with omega-regular objectives have been proposed~\cite{thomas95,ag11}. 
Reinforcement-learning (RL, for short) algorithms for partially-specified
Markov decision processes (MDPs) have been proposed (see
e.g.~\cite{wd92,DBLP:journals/jair/KaelblingLM96,rn10,sb18}) to learn
strategies that reach (near-)optimal performance in the actual
environment in which the system is executed. In this paper, we want to answer
the following question:
{\it How efficient can online-learning techniques be if only
correct executions, i.e. executions that satisfy a specified omega-regular
objective, are explored during execution?}
So, we want to understand how to combine synthesis and RL
to construct systems that are safe,
yet, at the same time, can adapt their behaviour according to the actual
environment in which they execute. 

\subparagraph{Problem statement.} In order to answer in a precise way the
question above, we consider a model halfway between the fully-unknown models
considered in RL and the fully-known models used in verification.  To be precise,
we consider as input an MDP with rewards whose transition probabilities are not
known and whose rewards are discovered on the fly. That is, the input is 
the support of the unknown transition function of the MDP.
This is natural from the point of view of verification since:
we may be working with an underspecified system,
its qualitative behaviour
may have already been observed,
or we may not trust all given probability values.
As optimization objective on
this MDP, we consider the mean-payoff function, and to capture the sure
omega-regular constraint we use a parity objective.

\subparagraph{Contributions.}
Given a lower bound $\pmin$ on the minimal transition probability,
we show that, in partially-specified
MDPs consisting of a single end component (EC), two
combinations of guarantees on the parity and mean-payoff objectives can be
achieved.
% first
(i)
For all $\epsilon$ and $\gamma$, we show how to construct a
finite-memory strategy which almost-surely satisfies the parity objective and
which achieves an $\epsilon$-optimal mean payoff with probability at least $1 -
\gamma$
(Prop.~\ref{pro:fmstrat}).
% second
(ii)
For all $\epsilon$ and $\gamma$, we show how to construct an
infinite-memory strategy which satisfies the parity objective
surely and which achieves an $\epsilon$-optimal mean payoff with probability at
least $1 - \gamma$
(Prop.~\ref{pro:fbstrat2}).
% end list
We also extend our results to MDPs consisting of more than one EC in a natural
way (Thms.~\ref{thm:case2} and~\ref{thm:case1}) and study special cases that
allow for improved optimality results as in the case of good ECs
(Props.~\ref{pro:fbstrat} and~\ref{pro:inf-mem-better}).  Finally, we show that
there are partially-specified MDPs for which stronger combinations of the
guarantees cannot be ensured.

Our usage of
$\pmin$ follows~\cite{DBLP:conf/atva/BrazdilCCFKKPU14,DBLP:conf/tacas/DacaHKP16}
where it is argued that
it is necessary for the statistical analysis of unbounded-horizon
properties and realistic in many scenarios. 

\subparagraph{Example: almost-sure constraints.}
%We illustrate in this example how to synthesize a finite-memory learning
%strategy that {\em almost-surely} wins the parity objective and
%ensures with {\em high probability} runs that are {\em near optimal} for the
%mean payoff.
Consider the MDP on the right-hand side of
Fig.~\ref{fig:opts} for which we know the support of the
transition function but not the probabilities $x$ and $y$ (for simplicity the
rewards are assumed to be known). 
First, note that while there is no surely winning strategy for the parity
objective in this MDP, playing action $a$ forever in $q_0$ guarantees to visit
state $q_3$ infinitely many times with probability one, i.e. this is a strategy
that almost-surely wins the parity objective.
Clearly, if $x > y$ then it is better to play
$b$ for optimizing the mean-payoff, otherwise, it is better to play $a$. As $x$
and $y$ are unknown, we need to learn estimates $\hat{x}$ and $\hat{y}$ for
those values to make a decision. This can be done by playing $a$ and $b$ a
number of times from $q_0$ and by observing how many times we get up and how
many times we get down. If $\hat{x} > \hat{y}$, we may choose to play $b$
forever in order to optimize our mean payoff.
We then face two difficulties.
First, after the learning episode, we may instead
observe $\hat{x} < \hat{y}$ while $x
\geq y$.
This is because we may have been unlucky and observed statistics that
differ from the real distribution. Second, playing $b$ always is not an option if
we want to satisfy the parity objective with probability $1$ (almost surely).
In this paper, we
give algorithms to overcome these two problems and compute a finite-memory strategy
that satisfies the parity objective with probability $1$ and is close to the optimal
expected mean-payoff value with high probability.

The finite-memory learning strategy produced by our algorithm works as follows
in this example.  First, it chooses $n \in \mathbb{N}$ large enough so that
trying $a$ and $b$ from $q_0$ as many as $n$ times allows it to learn $\hat{x}$ and
$\hat{y}$ such that $|\hat{x}-x| \leq \epsilon$ and $|\hat{y}-y| \leq \epsilon$
with probability at least $1-\gamma$. Then, if $\hat{x} > \hat{y}$ the strategy
plays $b$ for $K$ steps and then $a$ once. $K$ is chosen large enough
so that the mean payoff of any run will be $\epsilon$-close to the best
obtainable expected 
mean payoff with probability at least $1-\gamma$. Furthermore, as $a$
is played infinitely many times, the upper-right state will be visited
infinitely many times with probability $1$. Hence, the strategy is also
almost-surely satisfying the parity objective.

In the sequel we also show that if we allow for learning all along the execution
of the strategy then we can get, on this example, the exact optimal value and
satisfy the parity objective almost surely. However, to do so, we need infinite
memory.

\begin{figure}
    \begin{minipage}[t]{0.55\linewidth}
    \centering
    \begin{tikzpicture}
        \node[state,initial left](q0){$q_0:2$};
        \node[state,right= of q0](q1){$q_1:1$};
        \node[state,right= of q1](q2){$q_2:0$};
        \node[psplit,above=0.5cm of q1](q1a){ };

		\path
        (q0) edge[loop above] node[el]{$a:r_0$} (q0)
        (q0) edge node[el]{$b : 0$} (q1)
        (q1) edge[-] node[el,swap]{$a$} (q1a)
        (q1a) edge[bend right] node[el,swap,pos=0.25]{$1-x : 0$} (q0)
        (q1a) edge[bend left] node[el]{$x : r_1$} (q2)
        (q2) edge node[el,swap]{$a : r_1$} (q1)
        ;
	\end{tikzpicture}
    \end{minipage}
    \hfill
    \begin{minipage}[t]{0.4\linewidth}
    \centering
    \begin{tikzpicture}
        \node[state,initial above](q0) {$q_0 : 1$};
        \node[state,above right= of q0](qr1) {$q_3 : 0$};
        \node[state,below right= of q0](qr2) {$q_4 : 1$};
        \node[state,above left= of q0](ql1) {$q_1 : 1$};
        \node[state,below left= of q0](ql2) {$q_2 : 1$};
        \node[psplit,left=0.5cm of q0](ql) {};
        \node[psplit,right=0.5cm of q0](qr) {};

        \path
        (q0) edge[-] node[el]{$a$} (qr)
        (qr) edge[bend right] node[el,swap]{$x : 0$} (qr1)
        (qr1) edge[bend right] node[el]{$a : 0$} (q0)
        (qr) edge[bend left] node[el]{$1-x: 1$} (qr2)
        (qr2) edge[bend left] node[el,swap]{$a : 1$} (q0)
        (q0) edge[-] node[el,swap]{$b$} (ql)
        (ql) edge[bend left] node[el]{$y : 0$} (ql1)
        (ql1) edge[bend left] node[el,swap]{$a : 0$} (q0)
        (ql) edge[bend right] node[el,swap]{$1-y : 1$} (ql2)
        (ql2) edge[bend right] node[el]{$a : 1$} (q0)
        ;
    \end{tikzpicture}
    \end{minipage}
    \caption{Two automata, representing unknown MDPs, are depicted in the
        figure. Actions label edges from states (circles) to distributions
        (squares); a probability-reward pair, edges from distributions to
        states; an action-reward pair, Dirac
        transitions; a name-priority pair, states.}
    \label{fig:opts}
\end{figure}
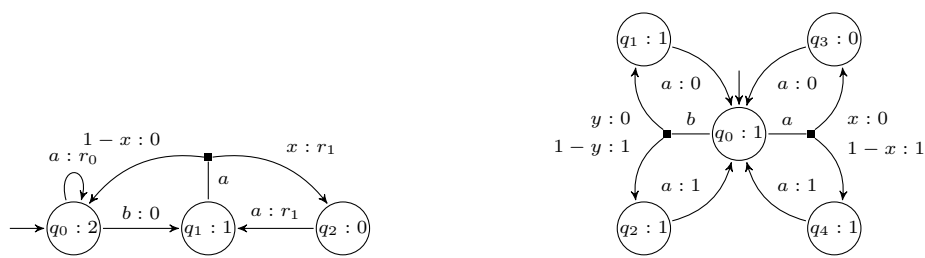

\subparagraph{Related works.}
In~\cite{bfrr14,DBLP:conf/lics/ClementeR15,brr17,DBLP:conf/aaai/Chatterjee0PRZ17},
we initiated the study of a mathematical model that combines MDPs and two-player
zero sum games.  With this new model, we provide formal grounds to synthesize
strategies that guarantee \textit{both} some minimal performance against any
adversary \textit{and} a higher expected performance against a given
expected behaviour of the environment, thus essentially combining the two
traditional standpoints from games and MDPs. Following this approach,
in~\cite{akv16}, Almagor et al.~study MDPs equipped with a mean-payoff and
parity objective. They study the problem of synthesizing a strategy that
ensures an expected mean-payoff value that is as large as possible while
satisfying a parity objective surely.  In~\cite{DBLP:conf/mfcs/ChatterjeeD11},
Chatterjee and Doyen study how to enforce almost surely a parity objective
together with a threshold constraint on the expected mean-payoff.  See
also~\cite{DBLP:conf/atva/BrazdilKN16}, where mean-payoff MDPs with
energy constraints are studied.  In all those works, the transition
probability and the reward function are {\em known} in advance. In contrast, we
consider the more complex setting in which the
reward function is {\em discovered on the fly during execution time} and the
transition probabilities need to be {\em learned}.

In~\cite{DBLP:conf/atva/DavidJLLLST14,DBLP:conf/iros/WenET15,jdtk16,DBLP:journals/corr/abs-1708-08611},
RL is combined with safety guarantees. In those works, there
is a MDP with a set of unsafe states that must be avoided at all cost. This MDP
is then restricted to states and actions that are safe and cannot lead to unsafe
states. Thereafter, classical RL is exercised. The problem that is
considered there is thus very similar to the problem that we study here with the
difference that they only consider {\em safety constraints}.
For safety constraints, the reactive synthesis phase and the
RL can be entirely decoupled with a two-phase algorithm.
A simple two-phase approach cannot be applied to the more general
setting of parity objectives. In our more challenging setting, we need to
intertwine the learning with the satisfaction of the parity objective in a non
trivial way. It is easy to show that reducing parity to safety, as
in~\cite{DBLP:journals/ita/BernetJW02}, could lead to learning strategies that
are arbitrary far from the optimal value that our learning strategies achieve.
In~\cite{DBLP:conf/ijcai/WenT16}, Topcu and Wen study how to learn in a MDP with
a discounted-sum (and not mean-payoff) function and liveness constraints
expressed as deterministic B\"uchi automata that must be enforced {\em almost
surely}.  Contrary to our setting, they do not consider general omega-regular
specifications expressed as parity objectives nor {\em sure} satisfaction. 

Finally, in~\cite{DBLP:conf/atva/BrazdilCCFKKPU14}, we apply RL to
MDPs where even the topology is unknown. Only $\pmin$ and, for convenience, the
size of the state space is given. There, we optimize the probability to satisfy an
omega-regular property; however, no mean payoff is involved. 

\subparagraph{Structure of the paper.} 
In Sect.~2, we introduce the necessary preliminaries.  In Sect.~3, we study
online finite and infinite-memory learning strategies for mean-payoff objectives
without omega-regular constraints.  In Sect.~4, we study 
strategies for mean-payoff objectives under a parity constraint that must be
enforced surely.  In Sect.~5, we study strategies for
mean-payoff objectives under a parity constraint that must be enforced almost
surely.

%% file: prelim.tex
%preliminaries.tex

Let $S$ be a finite set.  We denote by $\dist{S}$ the set of all
\emph{(rational) probabilistic distributions} on $S$, i.e. the set of all
functions $f : S \to \mathbb{Q}_{\ge 0}$ such that $\sum_{s \in S} f(s) = 1$.
For sets $A$ and $B$ and functions $g : A \to \dist{S}$ and $h : A \times B \to
\dist{S}$, we write $g(s|a)$ and $h(s|a,b)$ instead of $g(a)(s)$ and $h(a,b)(s)$
respectively.
The \emph{support} of a distribution $f \in \dist{S}$ is the set
\(
    \supp{f} \defequals \{ s \in S \st f(s) > 0 \}.
\)
The support of a function $g : A \to \dist{S}$ is the relation $R \subseteq A
\times S$ such that $(a,s) \in R \defiff g(s|a) > 0$.

\subsection{Markov chains}
\begin{definition}[Markov chains]
    A \emph{Markov chain} $\calC$ (MC, for short)
    is a tuple $(Q,\delta,p,r)$ where $Q$ is
    a (potentially countably infinite) set of states, $\delta$ is a
    (probabilistic) transition function $\delta : Q \to \dist{Q}$, $p : Q \to
    \mathbb{N}$ is a priority function, and $r : \supp{\delta} \to [0,1] \cap
    \mathbb{Q}$ is an (instantaneous) reward function.
\end{definition}
A \emph{run} of an MC is an infinite sequence of states $q_0 q_1 \dots
\in Q^\omega$ such that $\delta(q_{i+1}|q_i) > 0$ for all $0 \leq i$. We denote by
$\runs{\calC}{q_0}$ the set of all runs of $\calC$ that start with the state
$q_0$.

Consider an initial state $q_0$. The \emph{probability} of every measurable
\emph{event} $\calA \subseteq \runs{\calC}{q_0}$ is
well-defined~\cite{vardi85,puterman05}. We denote by $\probevent{\calC}{q_0}{\calA}$
the probability of $\calA$; for a measurable function $f : \runs{\calC}{q_0} \to
\mathbb{R}$, we write $\expect{\calC}{q_0}{f}$ for the \emph{expected value} of
the function $f$ under the probability measure $\probevent{\calC}{q_0}{\cdot}$
(see~\cite{norris98,puterman05} for a detailed
definition of these classical notions).

\subparagraph{Parity and mean payoff.} Consider a run
$\rho = q_0 q_1 \dots$ of $\calC$. We say $\rho$ \emph{satisfies the parity
objective}, written $\rho \models \parity$, if the minimal priority of states
along the run is even. That is to say
\(
    \rho \models \parity \defiff \liminf\{p(q_i) \st i \in \mathbb{N}\} \text{
    is even.}
\)
In a slight abuse of notation, we sometimes write $\parity$ to refer to the set
of all runs of an MC which satisfy the parity objective $\{\rho \in
\runs{\calC}{q_0} \st \rho \models \parity\}$. The latter set of runs is clearly
measurable.

The \emph{mean-payoff function} $\MP$ is defined for all runs $\rho = q_0
q_1\dots$ of $\calC$ as follows
\(
    \MP(\rho) \defequals \liminf_{j \in \mathbb{N}_{> 0}} \frac{1}{j}
    \sum_{i=0}^{j-1} r(q_i,q_{i+1}).
\)
This function is readily seen to be Borel
definable~\cite{chatterjee07}, thus also measurable.

\subsection{Markov decision processes}
\begin{definition}[Markov decision processes]
    A \emph{(finite discrete-time) Markov decision process} $\calM$ (MDP, for
    short) is a tuple $(Q,A,\alpha,\delta,p,r)$ where $Q$ is a finite set of
    states, $A$ a finite set of actions, $\alpha : Q \to \pow{A}$ a function
    that assigns to $q$ its set of available actions, $\delta : Q \times A \to
    \dist{Q}$ a (partial probabilistic) transition function with $\delta(q,a)$
    defined for all $q \in Q$ and all $a \in \alpha(q)$, $p : Q \to
    \mathbb{N}$ a priority function, and $r : \supp{\delta} \to [0,1]
    \cap \mathbb{Q}$ a reward function. We make the assumption that $\alpha(q)
    \neq \emptyset$ for all $q \in Q$, i.e. there are no deadlocks.
\end{definition}
A \emph{history} $h$ in an MDP is a finite state-reward-action sequence that
ends in a state and respects $\alpha$, $\delta$, and $r$, i.e.  if $h = q_0 a_0
x_0 \dots a_{k-1} x_{k-1} q_k$ then $a_i \in \alpha(q_i)$,
$\delta(q_{i+1}|q_i,a_i) > 0$, and $r(q_i,a_i,q_{i+1}) $, for all $0 \leq i <
k$. We write $\last(h)$ to denote the state $q_k$. For two histories $h,h'$, we
write $h < h'$ if $h$ is a \emph{proper prefix} of $h'$.

\begin{definition}[Strategies]
    A \emph{strategy} $\sigma$ in an MDP $\calM = (Q,A,\alpha,\delta,p,r)$ is a
    function $\sigma : (Q \cdot A \cdot \mathbb{Q})^*Q \to \dist{A}$ such that
    $\sigma(a|h) > 0 \implies a \in \alpha(\last(h))$.
\end{definition}
We write that a strategy $\sigma$ is \emph{memoryless} if $\sigma(h) =
\sigma(h')$ whenever $\last(h) = \last(h')$; \emph{deterministic} if for all
histories $h$ the distribution $\sigma(h)$ is Dirac.

Throughout this work we will speak of \emph{steps}, \emph{episodes}, and
\emph{following strategies}. We write that \emph{$\sigma$ follows $\tau$ (from
the history $h = q_0 a_0 x_0 \dots q_k$) during $n$ steps} if for all $h' = q'_0
a'_0 x'_0 \dots q'_{\ell}$, such that $h < h'$ and $\ell \leq k + n$, we have that
$\sigma(h') = \tau(h')$. An episode is simply a finite sequence of
consecutive steps, i.e. a
finite infix of the history, during which one or more strategies may have been
sequentially followed.

A \emph{stochastic Mealy machine} $\calT$ is a tuple $(M,m_0,f_u,f_o)$ where $M$
is a (potentially countably infinite) set of memory elements, $m_0 \in M$ is the
initial memory element, $f_u : M \times Q \times \mathbb{Q} \to M$ is an update
function, and $f_o : M \times Q \to \dist{A}$ is an output function. The machine
$\calT$ is said to implement a strategy $\sigma$ if for all histories $h = q_0
a_0 x_0 \dots a_{k-1} x_{k-1} q_k$ we have $\sigma(h) = f_o(m_k, q_k)$, where
$m_k$ is inductively defined as $m_i = f_u(m_{i-1}, q_{i-1}, x_{i-1})$ for all
$i \ge 1$.  It is easy to see that any strategy can be implemented by such a
machine.  A strategy $\sigma$ is said to have \emph{finite memory} if there
exists a stochastic Mealy machine that implements it and such that its set $M$
of memory elements is finite.

A (possibly infinite) state-action sequence $h = q_0 a_0 x_0 q_1 a_1 x_1 \dots$
is \emph{consistent with strategy $\sigma$} if $\sigma(a_i|q_0 a_0 x_0 \dots
a_{i-1} x_{i-1} q_{i}) > 0$ for all $i \geq 0$.

\subparagraph{From MDPs to MCs.}
The MDP $\calM$ and a strategy $\sigma$ implemented by the stochastic Mealy
machine $(M,m_0,f_u,f_0)$ induce the MC $\calM^\sigma = (Q',\delta',p',r')$
where $Q' = (Q \times M \times A) \cup (Q \times M)$; $\delta'(\langle q',m',a'
\rangle | s) = f_o(a'|m,q) \cdot \delta(q'|q,a')$ for any $s \in \{\langle q,m,a
\rangle, \langle q,m \rangle\}$ and $a' \in \alpha(q)$ with $(q,a',q') \in
\supp{\delta}$ and $m' = f_u(m,q,r(q,a',q'))$; $p'(\langle q,m,a\rangle) =
p'(\langle q,m\rangle) = p(q)$; and $r'(s,\langle q',m',a'\rangle) = r(q,a,q')$
for any $s \in \{ \langle q,m,a \rangle, \langle q,m\rangle\}$. For convenience,
we write $\probevent{\calM^\sigma}{q_0}{\cdot}$ instead of
$\probevent{\calM^\sigma}{\langle q_0, m_0 \rangle}{\cdot}$.

A strategy $\sigma$ is said to be \emph{unichain} if $\calM^\sigma$ has a single
recurrent class, i.e. a single bottom strongly-connected component (BSCC).

\subparagraph{End components.}
Consider a pair $(S,\beta)$ where $S
\subseteq Q$ and $\beta : S \to \pow{A}$ gives a subset of actions allowed per state
(i.e. $\beta(q) \subseteq \alpha(q)$ for all $q \in S$).
Let $\calG_{(S,\beta)}$ be the directed graph $(S,E)$ where $E$ is the set of
all pairs $(q,q') \in S \times S$
such that $\delta(q'|q,a) > 0$ for some $a \in \beta(q)$. We
say $(S, \beta)$ is an \emph{end component} (EC) if the following hold:
% first
if $a \in \beta(s)$, for $(s,a) \in S \times A$, then $\supp{\delta(s,a)}
\subseteq S$; and
% second
the graph $\calG_{(S,\beta)}$ is strongly connected.
% end list
Furthermore, we say the EC $(S,\beta)$ is \emph{good (for the parity
objective)} (a GEC, for short) if the minimal priority over all states from $S$ is
even; \emph{weakly good} if it contains a GEC.

For ECs $(S,\beta)$ and $(S',\beta')$, let us denote by $(S,\beta)
\subseteq (S',\beta')$ the fact that $S \subseteq S'$ and
\(
    \beta(s) \subseteq \beta'(s)
\)
for all $s \in S$.  We denote by $\mec{\calM}$ the set of all maximal ECs (MECs)
in $\calM$ with respect to $\subseteq$. It is easy to see that for all
$(S,\cdot),(S',\cdot) \in \mec{\calM}$ we have that $S \cap S' = \emptyset$,
i.e.  every state belongs to at most one MEC.

\subparagraph{Model learning and robust strategies.}
In this work we will ``approximate'' the stochastic dynamics of an unknown EC
in an MDP. Below, we formalize what we mean by
approximation.
\begin{definition}[Approximating distributions]
    Let $\calM = (Q,A,\alpha,\delta,p,r)$ be an MDP, $(S,\beta)$ an EC,
    and $\epsilon \in (0,1)$. We say $\delta'$ is
    $\epsilon$-close to $\delta$ in $(S,\beta)$, denoted $\delta'
    \eclose{\epsilon}{(S,\beta)} \delta$, if
    \(
        \left|\delta'(q'|q,a) - \delta(q'|q,a)\right| \leq \epsilon
    \)
    for all $q,q' \in S$ and all $a \in \beta(q)$. If the inequality holds
    for all $q,q' \in Q$ and all $a \in \alpha(q)$, then we write $\delta'
    \eclose{\epsilon}{ } \delta$.
\end{definition}

A strategy $\sigma$ is said to be \emph{(uniformly) expectation-optimal} if for
all $q_0 \in Q$ we have $\expect{\calM^\sigma}{q_0}{\MP} = \sup_\tau
\expect{\calM^\tau}{q_0}{\MP}$.  The following result captures the idea that some
expectation-optimal strategies for MDPs whose transition function have the same
support are ``robust''. That is, when used to play in another MDP with the same
support and close transition functions, they achieve near-optimal expectation.
\begin{lemma}[Follows from~{\cite[Theorem 6]{solan03}} and~{\cite[Theorem
    5]{chatterjee12}}]\label{lem:robust-opt}
    Consider values $\epsilon,\eta_\epsilon \in (0,1)$ such that
    \(
        \eta_\epsilon \leq \frac{\epsilon \cdot \pmin}{24|Q|};
    \)
    a transition function $\delta'$ such that $\supp{\delta} = \supp{\delta'}$
    and $\delta \eclose{\eta_\epsilon}{} \delta'$ where $\pmin$ is the minimal
    nonzero probability value from $\delta$ and $\delta'$; and a reward function
    $r'$ such that $\max\{|r(q,a,q') - r'(q,a,q')| : (q,a,q') \in
    \supp{\delta}\} \leq \frac{\epsilon}{4}$.
    For all memoryless deterministic expectation-optimal strategies $\sigma$ in
    $(Q,A,\alpha,\delta',p,r')$, for all $q_0 \in Q$, it holds that
    \(
        \left|
        \expect{\calM^\sigma}{q_0}{\MP}
        -
        \sup_{\tau}
        \expect{\calM^\tau}{q_0}{\MP}
        \right|
        \leq
        \epsilon.
    \)
\end{lemma}
We say a strategy $\sigma$ such as the one in the result above is
\emph{$\epsilon$-robust-optimal} (with respect to the expected mean payoff).

\subsection{Automata as proto-MDPs}
We study MDPs with unknown transition and reward functions. It is therefore
convenient to abstract those values and work with \emph{automata}.

\begin{definition}[Automata]
    A \emph{(finite-state parity) automaton} $\calA$ is a tuple
    $(Q,A,T,p)$ where $Q$ is a finite set of states, $A$ is a finite alphabet of
    actions, $T \subseteq Q \times A \times Q$ is a transition relation,
    and $p : Q \to \mathbb{N}$ is a priority function. We make the
    assumption that for all $q \in Q$ we have $(q,a,q') \in T$ for some
    $(a,q') \in A \times Q$.
\end{definition}

A transition function $\delta : Q \times A \to \dist{Q}$ is then
said to be \emph{compatible} with $\calA$ if 
$\forall (q,a) \in Q \times A: \supp{\delta(q,a)}=\{ q' \mid T(q,a,q')\}$.
For a transition
function $\delta$ compatible with $\calA$ and a reward function $r : 
T \to [0,1] \cap \mathbb{Q}$, we denote by $\autotomdp{\calA}{\delta}{r}$
the MDP $(Q,A,\alpha_T,\delta,p,r)$ where $a \in \alpha_T(q) \defiff \exists 
(q,a,q') \in T$. It is easy to see that the sets of ECs of MDPs
$(Q,A,\alpha_T,\delta,p,r)$ and $(Q,A,\alpha_T,\delta',p,r')$ coincide for all
$\delta'$ compatible with $\calA$ and all reward functions $r'$.  Hence, we will
sometimes speak of the ECs of an automaton.

\subparagraph{Example: sure-constraints.}
Consider the (variable-labelled-)automaton on the left-hand side of
Fig.~\ref{fig:opts}. Note that playing $a$ forever surely wins the parity
objective from everywhere but it may not be optimal for the expected mean
payoff.  To play optimally, we need to learn about the values $r_0$, $r_2$, and
$x$.  Assume that we play for $n$ steps $a$ and $b$ uniformly at random when in
state $q_0$. This will probably allow us to reach $q_1$ and $q_2$ a number of
times, and so to learn $r_0$ and $r_1$, and compute an estimate $\hat{x}$ of
$x$. If $\hat{x} \cdot r_1 > r_0$, we may want to conclude that the optimal
strategy is to always play $b$ from $q_0$. But we face here two major
difficulties.  First, after the learning episode of $n$ steps, we can observe
$\hat{x} \cdot r_1 > r_0$ while $x \cdot r_1 \leq r_0$, this is because we may
have been unlucky and observed statistics that differ from the real
distribution. Second, playing $b$ always is not an option if we want to surely
satisfy the parity objective. In this paper, we give algorithms to overcome the
two problems.  In our example, the strategy constructed by our algorithm will do
the following: given $\epsilon,\gamma \in (0,1)$, choose $n \in \mathbb{N}$
large enough, learn $\hat{x}$ such that $|\hat{x}-x| \leq \epsilon$ with
probability more than $1-\gamma$, then if $\hat{x} \cdot r_1 \leq r_0$, play $a$
forever. Otherwise, keep playing $b$ for longer and longer episodes.  If during
one of these episodes, the state $q_2$ is not visited (i.e. the parity objective
is endangered as the minimal priority seen during the episode is odd) switch to
playing $a$ forever. 

\subparagraph{Transition-probability lower bound.}
Let $\pmin \in [0,1] \cap \mathbb{Q}$ be a \emph{transition-probability lower
bound}. We say that $\delta$ is \emph{compatible} with $\pmin$ if for all
$(q,a,q') \in Q \times A \times Q$ we have that: either $\delta(q'|q,a) \ge
\pmin$ or $\delta(q'|q,a) = 0$.

%% file: mp.tex
%mp.tex
\label{sec:mp}
In this section, we focus on the design of optimal learning strategies for the
mean-payoff function in the unconstrained single-end-component case. That is, we
have an unknown strongly connected MDP with no parity objective.

We consider, in turn, learning strategies that use finite and infinite memory.
Whereas classical RL algorithms focus on achieving an optimal expected value
(see, e.g.,~\cite{wd92}; cf.~\cite{bdm17}), we prove here that a stronger result
is achievable: one can ensure---using finite memory only---outcomes that are
close to the best expected value with high probability. Further, with infinite
memory the optimal outcomes can be ensured with probability $1$. In both cases,
we argue that our results are tight.

For the rest of this section, let us fix an automaton $\calA = (Q,A,T,p)$ such
that $(Q,\alpha_T)$ is an EC, and some $\pmin \in (0,1]$.

\subparagraph{Yardstick.}
Let $\delta$ be a transition function compatible with $\calA$ and $\pmin$, and 
$r$ be a reward function. The optimal expected mean-payoff value that is achievable in
the unique EC $(Q,\alpha_T)$ is defined as
\(
    \Val(Q,\alpha_T) \defequals \sup_\sigma
    \expect{\autotochain{\calA}{\delta}{r}{\sigma}}{q_0}{\MP}
\)
for any $q_0 \in Q$. Indeed, it is well known that this value
is the same for all states in the same
EC.

Note that this value can always be obtained by a memoryless
deterministic~\cite{gimbert07} and unichain~\cite{bfrr14} expectation-optimal
strategy when $\delta$ and $r$ are known. We will use this value as a
yardstick for measuring the performance of the learning strategies we describe
below.

\subparagraph{Model learning.}
Our strategies learn approximate models of $\delta$ and $r$
to be able to compute near-optimal strategies. To
obtain those models, we use an approach based on ideas from probably
approximately correct (PAC) learning. Namely, we will execute a
random exploration of the MDP for some number of steps and obtain an empirical
estimation of its stochastic dynamics, see e.g.~\cite{valiant84}. We say that a
memoryless strategy $\lambda$ is a \emph{(uniform random) exploration strategy}
for a function $\beta : Q \to \pow{A}$ if 
\(
    \lambda(a|q) = 1/|\beta(q)|
\)
for all $q \in Q, a \in \alpha(q)$ such that
$a \in \beta(q)$ and $\lambda(a|q) = 0$ otherwise.
Each time the random exploration enters a state $q$ and
chooses an action $a$, we say that it performs an experiment on $(q,a)$, and if
the state reached is $q'$ then we say that the result of the experiment is $q'$.
Furthermore, the value $r(q,a,q')$ is then known to us. To learn an
approximation $\delta'$ of the transition function $\delta$, and to learn $r$,
the learning strategy remembers statistics about such experiments. If the random
exploration strategy is executed long enough then it collects sufficiently many
experiment results to accurately approximate the transition function $\delta$
and the exact reward function $r$ with high probability.

The next lemma gives us a bound on the number of $|Q|$-step episodes for which
we need to exercise such a strategy to obtain the desired approximation with at
least some given probability. It can be proved via a simple application of
Hoeffding's inequality.
\begin{lemma}\label{lem:suff-urandom}
    For all ECs $(S,\beta)$ and all
    $\epsilon, \gamma \in (0,1)$ one can compute $n \in
    \mathbb{N}$ (exponential in $|Q|$ and polynomial in $|A|$, $\pmin^{-1}$,
    $\ln(\gamma^{-1})$, and $\epsilon^{-1}$) such that following an
    exploration strategy for $\beta$ during $n$ (potentially
    non-consecutive) episodes of $|Q|$-steps suffices
    to collect enough information to be able to compute a transition function
    $\delta'$ such that
    \(
        \rvprob{\delta' \eclose{\epsilon}{(S,\beta)} \delta} \geq 1 - \gamma.
    \)
\end{lemma}

\newcommand{\fmstrat}{\sigma_\mathrm{fin}}

\subsection{Finite memory}\label{sec:mp-fm}
We now present a family of finite memory strategies $\fmstrat$ that force, given
any $\epsilon,\gamma \in (0,1)$, outcomes with a mean payoff that is
$\epsilon$-close to the optimal expected value with probability higher than
$1-\gamma$. The strategy  $\fmstrat$ is defined as follows.
\begin{enumerate}
    \item First, $\fmstrat$ follows the model-learning strategy above for $L$
        steps, according to Lemma~\ref{lem:suff-urandom}, in order to obtain an
        approximation $\delta'$ of $\delta$ such that $\delta'
        \eclose{\eta}{ } \delta$ with probability at least $1 - \gamma$.
        A reward function $r'$ is also constructed from the
        observed rewards.
    \item Then, $\fmstrat$ follows a unichain memoryless deterministic
        expectation-optimal strategy $\tau$ for
        $\autotomdp{\calA}{\delta'}{r'}$. 
\end{enumerate}
The following result tells us that if the learning phase is sufficiently long,
then we can obtain, with $\fmstrat$, a near-optimal
mean payoff with high probability.

\begin{proposition}
    For all $\epsilon, \gamma \in (0,1)$, one can compute $L \in \mathbb{N}$
    such that for the resulting finite-memory strategy $\fmstrat$, for all
    $q_0 \in Q$, for all $\delta$ compatible with $\calA$ and $\pmin$, and for
    all reward functions $r$, we have
    \(
        \probevent{\autotochain{\calA}{\delta}{r}{\fmstrat}}{q_0}{
            \rho : 
            \MP(\rho) \geq \Val(Q,\alpha_T) - \epsilon
            } \ge 1 - \gamma.
    \)
\end{proposition}
\begin{proof}
    We will make use of Lemma~\ref{lem:robust-opt}. For that purpose, let $\eta
    = \min\{\pmin, \eta_\epsilon\}$ where $\eta_\epsilon$ is as in the
    statement of the lemma. Next, we set $L = |Q|n$ where $n$ is as dictated by
    Lemma~\ref{lem:suff-urandom} using $\eta$ and $\gamma$. By
    Lemma~\ref{lem:suff-urandom}, with probability at least $1 - \gamma$ our
    approximation $\delta'$ is such that $\delta' \eclose{\eta}{ } \delta$. 
    Since $\eta \leq \pmin$, it follows that $\supp{\delta} =
    \supp{\delta'}$ and we now have learned $r$, again
    with probability $1 - \gamma$. Finally, since $\eta \leq \eta_\epsilon$ and
    $\tau$ is a unichain memoryless deterministic expectation-optimal strategy,
    Lemmas~\ref{lem:robust-opt} and~\ref{lem:tracol} item~\ref{itm:lim} imply
    the desired result.
\end{proof}

\begin{remark}[Finite-memory implementability]
    Note that $\fmstrat$, as we described it previously, is not immediately seen
    to be a computable finite stochastic Mealy machine. Let us consider all
    possible histories of length $L$. Observe that this set is not finite
    because of the unknown rewards which can range over arbitrary rational
    numbers in $[0,1]$. However, we can finitize the set by focusing only on
    rewards of bounded representation size by imposing an upper-bound on the
    bitsize of their representation (truncating the rest off observed rewards)
    while still satisfying the hypotheses of Lemma~\ref{lem:robust-opt}.  Now,
    for all such histories we can compute an approximation $\delta'$ of $\delta$
    and an approximation $r'$ of the observed reward function $r$.  Using that
    information, the required finite-memory expectation-optimal strategy $\tau$
    can be computed. We encode these (finitely many) strategies into the
    machine implementing $\fmstrat$ so that it only has to choose which one to
    follow forever after the (finite) learning phase has ended.  Hence, one can
    indeed construct a finite-memory strategy realizing the described strategy.
\end{remark}

\subparagraph{Optimality.}
The following tells us that we cannot do better with finite memory
strategies.

\begin{proposition}\label{pro:opt-fin-mem}
    Let $\calA$ be the single-EC automaton on the right-hand side of
    Fig.~\ref{fig:opts} and $\pmin \in (0,1]$.
    For all $\epsilon,\gamma \in (0,1)$, the following two statements hold.
    \begin{itemize}
        \item For all finite memory strategies $\sigma$, there exist $\delta$
            compatible with $\calA$ and $\pmin$,
            and a reward function $r$, such that
            \(
                \probevent{\autotochain{\calA}{\delta}{r}{\sigma}}{q_0}{
                    \rho : 
                    \MP(\rho) \geq \Val(Q,\alpha_T) - \epsilon
                    } < 1.
            \)
        \item For all finite memory strategies $\sigma$, there exist $\delta$
            compatible with $\calA$ and $\pmin$, and a reward function
            $r$ such that
            \(
                \probevent{\autotochain{\calA}{\delta}{r}{\sigma}}{q_0}{
                    \rho : 
                    \MP(\rho) < \Val(Q,\alpha_T)
                    } \ge \gamma.
            \)
    \end{itemize}
\end{proposition}
\begin{proof}[Proof sketch]
    With a finite-memory strategy we cannot satisfy a stronger guarantee than
    being $\epsilon$-optimal with probability at least $1 - \gamma$ in this
    example.
    Indeed, as we can only use finite memory, we can only learn imprecise models
    of $\delta$ and $r$. That is, we will always have a non-zero probability to
    have approximated $x$ or $y$ arbitrarily far from their actual values. It
    should then be clear that neither optimality with high probability nor
    almost-sure $\epsilon$-optimality can be achieved.
\end{proof}

\newcommand{\imstrat}{\sigma_\infty}

\subsection{Infinite memory}
While we have shown that probably approximately optimal is the best that can
be obtained with finite memory learning strategies, we now establish that with
infinite memory, one can guarantee almost sure optimality.

To this end, we define a strategy $\imstrat$ which operates in episodes
consisting of two phases: learning and optimization. In episode $i \in
\mathbb{N}$, the strategy does the following.
\begin{enumerate}
    \item It first follows an exploration strategy $\lambda$ for $\alpha_T$
        during $L_i$ steps, there exist models $\delta_i$ and $r_i$ based on the
        experiments obtained throughout the $\sum_{j=0}^i L_j$ steps during
        which $\lambda$ has been followed so far.
    \item Then, $\imstrat$ follows a unichain memoryless
        deterministic expectation-optimal strategy
        $\sigma_{\mathrm{MP}}^{\delta_i}$ for $\autotomdp{\calA}{\delta_i}{r_i}$
        during $O_i$ steps.
\end{enumerate}
One can then argue that $\imstrat$ can be instantiated so that in every episode
the finite average obtained so far gets ever close to $\Val(Q,\alpha_T)$ with
ever higher probability. This is achieved by choosing the $L_i$ as an increasing
sequence so that the approximations $\delta_i$ get ever better with ever higher
probability. Then, the $O_i$ are chosen so as to compensate
% first
for the past history,
% second
for the time before the induced MC reaches its limit
distribution, and
% third
for the future number of steps that will be spent learning in the next episode.
% end list
The latter then allows us to use the Borel-Cantelli lemma to show that in
the unknown EC we can obtain its value almost surely. 
\begin{proposition}\label{pro:mp-opt-ec}
    One can compute a sequence $(L_i,O_i)_{i \in \mathbb{N}}$ such that $L_i
    \geq |Q|$ for all $i \in \mathbb{N}$; additionally the resulting strategy
    $\imstrat$ is such that for all $q_0 \in Q$, for
    all $\delta$ compatible with $\calA$ and $\pmin$, and for all reward
    functions $r$, we have
    \(
        \probevent{\autotochain{\calA}{\delta}{r}{\imstrat}}{q_0}{
        \rho : \MP(\rho) \ge \Val(Q,\alpha_T)
        } = 1.
    \)
\end{proposition}

\subparagraph{Optimality.} Note that $\imstrat$ is
optimal since it obtains with probability $1$ the best value that can be
obtained when the MDP is fully known, i.e. when $\delta$ and $r$ are known in
advance.

%% file: sure.tex
%sure.tex

\newcommand{\parstrat}{\sigma_{\mathrm{par}}}

We show here how to design learning strategies that obtain near-optimal
mean-payoff values while ensuring that all runs satisfy a given parity objective
with certainty. 

First, we note that all such learning strategies must avoid entering states $q$
from which there is no strategy to enforce the parity objective with certainty.
Hence, we make the hypothesis that all such states have been removed from the
automaton $\calA$, and so we assume that for all $q_0 \in Q$ there exists a
strategy $\parstrat$ such that for all functions $\delta$ compatible with
$\calA$, for all reward functions $r$, and for all $\rho \in
\runs{\autotochain{\calA}{\delta}{r}{\sigma}}{q_0}$, we have $\rho \models
\parity$.  It is worth noting that, in fact, there exists a memoryless
deterministic strategy such that the condition holds for all $q_0 \in
Q$~\cite{ag11,ar17}.  Notice the swapping of the quantifiers over the initial
states and the strategy, this is why we say it is \emph{uniformly winning
(for the parity objective)}.  The set of states to be removed, along with a
uniformly winning strategy, can be computed in quasi-polynomial
time~\cite{cjkls17}. We say that an automaton with no states from which there is
no winning strategy is \emph{surely good}.

We study the design of learning strategies for mean-payoff optimization under
{\em sure} parity constraints for increasingly complex cases.

\subsection{The case of a single good end component}
Consider a surely-good automaton $\calA = (Q,A,T,p)$ such that $(Q,\alpha_T)$ is
a GEC, i.e. the minimal priority of a state in the EC is even, and some $\pmin \in
(0,1]$.

\subparagraph{Yardstick.}
For this case, we use as yardstick the optimal expected mean-payoff value:
\(\Val(Q,\alpha_T) =\sup_\sigma
\expect{\autotochain{\calA}{\delta}{r}{\sigma}}{q_0}{\MP}.\)

\subparagraph{Learning strategy.}
We show here that it is possible to obtain an optimal mean-payoff
with high probability. Note that our solution extends a result given by Almagor et
al.~\cite{akv16} for {\em known} MDPs. The main idea behind our solution is to
use the strategy $\imstrat$ from Proposition~\ref{pro:mp-opt-ec} in a
controlled way: we verify that during all successive learning and optimization
episodes, the minimal parity value that is visited is even. If during some
episode, this is not the case, then we resort to a strategy $\parstrat$ that
enforces the parity objective with certainty. Such $\parstrat$ is guaranteed to
exist as $\calA$ is surely good. 

\begin{proposition}\label{pro:fbstrat}
    For all $\gamma \in (0,1)$, there exists a strategy $\sigma$ such that for
    all $q_0 \in Q$, for all $\delta$ compatible with $\calA$ and $\pmin$, and
    for all reward functions $r$, we have
    % parity
        $\rho \models \parity$ for all $\rho \in
            \runs{\autotochain{\calA}{\delta}{r}{\sigma}}{q_0}$ and
    % mean-payoff
        $\probevent{\autotochain{\calA}{\delta}{r}{\sigma}}{q_0}{
                        \rho 
                        :
                        \MP(\rho) \ge \Val(Q,\alpha_T)
                } \ge 1 - \gamma$.
\end{proposition}
\begin{proof}[Proof sketch]
    We modify $\imstrat$ so as to ``give up'' on optimizing the mean
    payoff if the minimal even priority has not been seen during a long
    sequence of episodes.  This will guarantee that the measure of runs which
    give up on the mean-payoff optimization is at most $\gamma$.

    First, recall that we can instantiate $\imstrat$ so that $L_i \geq |Q|$ for
    all $i \in \mathbb{N}$. Hence, with some probability $\zeta > 0$, during
    every learning phase, we visit a state with even minimal priority. We can
    then find a sequence $n_1, n_2, \dots \in \mathbb{N}^\omega$ of natural
    numbers such that $\prod^\infty_{j = i} (1 - \zeta^{n_j}) \ge 1 - \gamma$,
    for some $i \in \mathbb{N}$. Given this sequence, we apply the following
    monitoring. If for $\ell \in \mathbb{N}$ we write $N_\ell \defequals
    \sum_{k = 1}^{\ell - 1} n_k$, then at the end of the $\ell$-th episode we
    verify that during some learning phase from $L_{N_\ell}, L_{N_\ell + 1},
    \dots, L_{N_\ell + n_\ell}$ we have visited a state with minimal even
    priority, otherwise we switch to a parity-winning strategy forever.
\end{proof}

\subparagraph{Optimality.}

The following proposition tells us that the guarantees from 
Proposition~\ref{pro:fbstrat} are indeed optimal w.r.t. our chosen yardstick.

\begin{proposition}
    Let $\calA$ be the single-GEC automaton on the left-hand side of
    Fig.~\ref{fig:opts} and 
    $\pmin \in (0,1]$.  For all 
    parity-winning strategies $\sigma$, there exist $\delta$
    compatible with $\calA$ and $\pmin$, and a reward function $r$, such that
    \(
        \probevent{\autotochain{\calA}{\delta}{r}{\sigma}}{q_0}{
            \rho : 
            \MP(\rho) \geq \Val(Q,\alpha_T)
            } < 1.
    \)
\end{proposition}
\begin{proof}[Proof sketch]
    Consider a reward function such that $r_0 = 0$ and $r_1 = 1$ and an
    arbitrary $\delta$.  It is easy to see that $\Val(Q,\alpha_T) = 1$. However,
    any strategy that ensures the parity objective is satisfied surely must be
    such that, with probability $\gamma > 0$, it switches to follow a
    strategy $q_2 \mapsto (a \mapsto 1)$ forever. Hence, with probability at
    least $\gamma$ its mean-payoff is sub-optimal.
\end{proof}

\subsection{The case of a single end component}
We now turn to the case where the surely-good automaton $\calA = (Q,A,T,p)$
consists of a unique, not necessarily good, EC $(Q,\alpha_T)$. Let us also fix
some $\pmin \in (0,1]$.

An important observation regarding single-end-component MDPs that are surely
good is that they contain at least one GEC as stated in the following lemma.
\begin{lemma}\label{lem:contain-good-ec}
    For all surely-good automata $\calA = (Q,A,T,p)$ such that
    $(Q,\alpha_T)$ is an EC there exists $(S,\beta) \subseteq
    (Q,\alpha_T)$ such that $(S,\beta)$ is a GEC in
    $\autotomdp{\calA}{\delta}{r}$ for all $\delta$ compatible with $\calA$ and
    all reward functions $r$, i.e. $(Q,\alpha_T)$ is weakly good.
\end{lemma}

\subparagraph{Yardstick.} Let $\delta$ and $r$ be fixed in the single EC, our
yardstick for this case is defined as follows:
\(
    \cVal(Q,\alpha_T) \defequals
    \max_{q \in Q} \sup\left\{
    \expect{\calA^\sigma_{\delta,r}}{q}{\MP}
    \:\middle|\:
    \sigma \text{ is a parity-winning strategy}
    \right\}.
\)
That is $\cVal(Q,\alpha_T)$ is the best MP expectation value that can be
obtained from a state $q \in Q$ with a parity-winning strategy. It is remarkable
to note that we take the maximal value over all states in $Q$. As noted by
Almagor et al.~\cite{akv16}, this value is not always achievable even when
$\delta$ and $r$ are a priori known, but it can be approached arbitrarily close.

\subparagraph{Learning strategy.}
The following proposition tells us that we can obtain a value close to
$\cVal(Q,\alpha_T)$ with arbitrarily high probability while satisfying the
parity objective surely.
\begin{proposition}\label{pro:fbstrat2}
    For all $\epsilon,\gamma \in (0,1)$ there exists a strategy
    $\sigma$ such that for all $q_0 \in Q$,
    for all $\delta$ compatible with $\calA$ and
    $\pmin$, and for all reward functions $r$, we have
    % parity
        $\rho \models \parity$ for all $\rho \in
            \runs{\autotochain{\calA}{\delta}{r}{\sigma}}{q_0}$ and
    % mean-payoff
        $\probevent{\autotochain{\calA}{\delta}{r}{\sigma}}{q_0}{
                \rho : \MP(\rho) \geq \cVal(Q,\alpha_T) - \epsilon
            } \geq 1 - \gamma$.
\end{proposition}
\begin{proof}[Proof sketch]
    We define a strategy $\sigma$ as follows. Let $\eta = \min\{\pmin,
    \eta_{\epsilon/2}\}$ for $\eta_{\epsilon/2}$ as defined for
    Lemma~\ref{lem:robust-opt}. The strategy $\sigma$ plays as follows.
    \begin{enumerate}
        \item It first computes $\delta'$ such that $\delta' \eclose{\eta}{ }
            \delta$ with probability at least $1 - \gamma/4$ and a reward
            function $r'$ by following an exploration strategy $\lambda$ for
            $\alpha_T$ during $J_0$ steps (see Lemma~\ref{lem:suff-urandom}).
        \item It then selects a contained good MEC with maximal
            expected mean-payoff
            value (see Lemma~\ref{lem:contain-good-ec}) and tries to reach it
            with probability at least $1 - \gamma/4$ by following $\lambda$
            during $J_1$ steps.
        \item Finally, if the component is reached, it follows a strategy
            $\tau$ as in Proposition~\ref{pro:fbstrat} with $\gamma/4$ from then
            onward.
    \end{enumerate}
    If the learning ``fails'' or if the component is not reached, the strategy
    reverts to following a winning strategy forever. (A failed learning phase is
    one in which the approximated distribution function does not have $T$ as its
    support. The EC-reaching phase may also fail.)
\end{proof}

\subparagraph{Optimality.} The following states that we cannot
improve on the result of Proposition~\ref{pro:fbstrat2}.

\begin{figure}
    \centering
    \begin{tikzpicture}
        \node[state,initial above](q0){$q_0 : 1$};
        \node[state,left= of q0](q1){$q_1 : 2$};
        \node[state,left= of q1](q2){$q_2 : 2$};
        \node[state,right= of q0](q3){$q_3 : 2$};
        \node[state,right= of q3](q4){$q_4 : 2$};
        \node[psplit,above=0.5cm of q2](q2a){ };
        \node[psplit,above=0.5cm of q1](q1a){ };
        \node[psplit,above=0.5cm of q3](q3a){ };
        \node[psplit,above=0.5cm of q4](q4a){ };

        \path
        (q0) edge[bend right] node[el,swap]{$a$} (q1)
        (q1) edge[bend right] node[el]{$b$} (q0)
        (q0) edge[bend right] node[el]{$b$} (q3)
        (q3) edge[bend right] node[el,swap]{$a$} (q0)
        (q2) edge[-] node[el,swap]{$a$} (q2a)
        (q2a) edge[bend right] (q2)
        (q2a) edge[bend left] (q1)
        (q1) edge[-] node[el]{$a$} (q1a)
        (q1a) edge[bend left] (q1)
        (q1a) edge[bend right] (q2)
        (q3) edge[-] node[el,swap]{$b$} (q3a)
        (q3a) edge[bend right] (q3)
        (q3a) edge[bend left] (q4)
        (q4) edge[-] node[el]{$b$} (q4a)
        (q4a) edge[bend left] (q4)
        (q4a) edge[bend right] (q3)
        ;
    \end{tikzpicture}
    \caption{An automaton for which it is impossible to learn to obtain
    near-optimal mean-payoff almost surely or optimal mean-payoff with high
    probability, while satisfying the parity objective. For clarity, probability
    and reward placeholders have been omitted.}
    \label{fig:opt3}
\end{figure}
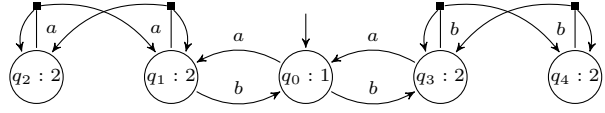

\begin{proposition}\label{prop:opt-sure-ec}
    Let $\calA$ be the single-EC automaton in Fig.~\ref{fig:opt3} and $\pmin \in
    (0,1]$.
    For all $\epsilon,\gamma \in (0,1)$, the two following statements hold.
	\begin{itemize}
        \item For all strategies $\sigma$, there exist $\delta$ compatible with
            $\calA$ and $\pmin$, and a reward function $r$, such
            that \(
                \probevent{\autotochain{\calA}{\delta}{r}{\sigma}}{q_0}{
                    \rho :
                    \MP(\rho) \ge \cVal(Q,\alpha_T) - \epsilon
                } < 1.
            \)
        \item For all strategies $\sigma$, there exist $\delta$ compatible with
            $\calA$ and $\pmin$, and a reward function $r$, such
            that \(
               \probevent{\autotochain{\calA}{\delta}{r}{\sigma}}{q_0}{
                    \rho:
                    \MP(\rho) \ge \cVal(Q,\alpha_T)
                } < 1 - \gamma.
            \)
	\end{itemize}
\end{proposition}
\begin{proof}[Proof sketch]
    Observe that the MEC is not a good EC. However, it does contain the
    GECs with states $\{q_1,q_2\}$ and $\{q_3,q_4\}$ respectively. Now, since
    those two GECs are separated by $q_0$, whose priority is $1$, any winning
    strategy must at some point stop playing to $q_0$ and commit to a single
    GEC. Thus, the learning of the global EC can only last for a finite number
    of steps. It is then straightforward to argue that near-optimality with
    high-probability is the best achievable guarantee.
\end{proof}

\subsection{General surely-good automata}
\label{sec:sure-gen-auto}
In this section, we generalize our approach from single-EC automata to general
automata. We will argue that, under a sure parity constraint, we can achieve a
near-optimal mean payoff with high probability in any MEC $(S,\beta)$ in which we
end up with non-zero probability. That is, given that the event $\Inf \subseteq
S$, defined as the set of all runs that eventually always stay within $S$, has
non-zero probability measure.

\begin{theorem}\label{thm:case1}
    Consider a surely-good automaton $\calA = (Q,A,T,p)$
    and some $\pmin \in (0,1]$. For all
    $\epsilon,\gamma \in (0,1)$ there exists
    a strategy $\sigma$ such that for all $q_0 \in Q$,
    for all $\delta$ compatible with
    $\calA$ and $\pmin$, and all reward functions $r$, we have
    \begin{itemize}
        \item $\rho \models \parity$ for all $\rho \in
            \runs{\autotochain{\calA}{\delta}{r}{\sigma}}{q_0}$ and
        \item
            \(
                \probevent{\autotochain{\calA}{\delta}{r}{\sigma}}{q_0}{
                        \rho
                        :
                        \MP(\rho) \ge \cVal(S,\beta) - \epsilon
                    \:\middle|\:
                    \Inf \subseteq S
                } \geq 1 - \gamma
            \)
            for all $(S,\beta) \in \mec{\autotomdp{\calA}{\delta}{r}}$
            such that $(S,\beta)$ is weakly good and 
            $\probevent{\autotochain{\calA}{\delta}{r}{\sigma}}{q_0}{\Inf
            \subseteq S} > 0$.
    \end{itemize}
\end{theorem}
\begin{proof}[Proof sketch]
    The strategy $\sigma$ we construct follows a parity-winning
    strategy $\parstrat$ until a state contained in a weakly good MEC,
    that has not been visited before, is entered. In this case, the strategy
    follows $\tau$ (the strategy from Proposition~\ref{pro:fbstrat2}). Observe
    that when $\tau$ switches to $\parstrat$ (a parity-winning strategy) it may
    exit the end component. If this happens, then the component is marked as
    visited and $\parstrat$ is followed until a new---not previously
    visited---maximal good end component is entered. In that case, we switch to
    $\tau$ once more.  Crucially, the new strategy $\sigma$ ignores MECs that
    are revisited 
\end{proof}

\begin{remark}[On the choice of MECs to reach]
    The strategy constructed for the proof of Theorem~\ref{thm:case1} has to
    deal with leaving a MEC due to the fallbacks to the parity-winning strategy
    $\parstrat$.
    However, surprisingly, instead of actually following $\parstrat$,
    upon entering a new MEC it has to restart the process of achieving
    a satisfactory mean-payoff. Indeed, otherwise the overall mass of
    sub-optimal runs from various MECs (each smaller than $\gamma$) could get
    concentrated in a single MEC, thus violating the advertised guarantees.

    The strategy could be simplified
    as follows. First, we follow any strategy to reach a bottom MEC
    (BMEC)---that is, a MEC from which no other MEC is reachable. By definition,
    the winning strategy can be played here and the MEC cannot be escaped.
    Therefore, in the BMEC we run the strategy as described, and after the
    fallback we indeed simply follow $\parstrat$.  If we did not
    reach a BMEC after a long time, we could switch to the fallback, too.
    While this strategy is certainly simpler, our general strategy has the
    following advantage.
    Intuitively, we can force the strategy to stay in any current good MEC, even
    if it is not bottom, and thus maybe achieve a more satisfactory mean-payoff.
    Further, whenever we want, we can force the strategy to leave the current
    MEC and go to a lower one. For instance, if the current estimate of the mean
    payoff is lower than what we hope for, we can try our luck in a lower MEC.
    We further comment on the choice of unknown MECs in the conclusions.
\end{remark}

%% file: almost.tex
%almost.tex

In this section, we turn our attention to learning strategies that must ensure a
parity objective not with certainty (as in previous section) but {\em almost
surely}, i.e. with probability $1$.  As winning almost surely is less stringent,
we can hope both for a stricter yardstick (i.e. better target values) and also
better ways of achieving such high values. We show here that this is
indeed the case. Additionally, we argue that several important learning results
can now be obtained with finite-memory strategies.

As previously, we make the hypothesis that we have removed from $\calA$ all
states from which the parity objective cannot be forced with probability $1$ (no
such state can ever be entered). Note that to compute the set of states to
remove, we do not need the knowledge of $\delta$ but only the support as given
by $\calA$. States to remove can be computed in polynomial time using
graph-based algorithms~(see, e.g.,~\cite{bk08}). An automaton $\calA$ which
contains only almost-surely winning states for the parity objective  is called
\emph{almost-surely good}.

\newcommand{\asparstrat}{\sigma^{\mathrm{as}}_{\mathrm{par}}}

We have, as in the previous section,
that for all automata $\calA$ there exists a memoryless
deterministic strategy $\sigma$ such that for all $q_0 \in Q$, for all $\delta$
compatible with $\calA$, for all $r$, the measure of the subset of $\rho \in
\runs{\autotochain{\calA}{\mu}{r}{\sigma}}{q_0}$ such that $\rho \models
\parity$ is equal to $1$ (see e.g.~\cite{bk08}). Such a strategy is said to be
\emph{uniformly almost-sure winning (for the parity objective)}. In the sequel,
we denote such a strategy $\asparstrat$.

We now study the design of learning strategies for mean-payoff optimization
under {\em almost-sure} parity constraints for increasingly complex cases.

\subsection{The case of a single good end component}

Consider an automaton $\calA = (Q,A,T,p)$ such that $(Q,\alpha_T)$ is a
GEC, and some $\pmin \in (0,1]$.

\subparagraph{Yardstick.}
For this case, we use as a yardstick the optimal expected mean-payoff value:
\(
    \Val(Q,\alpha_T) =
    \sup_\sigma \expect{\autotochain{\calA}{\delta}{r}{\sigma}}{q_0}{\MP}
\)
for any $q_0 \in Q$.
%Note that this is the same yardstick as for the sure case. 

\subparagraph{Learning strategies.}
We start by noting that $\imstrat$ from
Section~\ref{sec:mp} also ensures that the parity objective is satisfied almost
surely when exercised in a GEC.
\begin{proposition}\label{pro:inf-mem-better}
    One can compute a sequence $(L_i,O_i)_{i \in \mathbb{N}}$ such that for the
    resulting strategy $\imstrat$ we have that
    for all $q_0 \in Q$, for
    all $\delta$ compatible with $\calA$ and $\pmin$, and for all reward
    functions $r$, we have
    % parity
        $\probevent{\autotochain{\calA}{\delta}{r}{\imstrat}}{q_0}{\parity}
            = 1$ and
    % mean payoff
        $\probevent{\autotochain{\calA}{\delta}{r}{\imstrat}}{q_0}{
                \rho : \MP(\rho) \ge \Val(Q,\alpha_T)
            } = 1$.
\end{proposition}
\begin{proof}
    By Proposition~\ref{pro:mp-opt-ec}, one can choose parameter
    sequences such that $L_i \geq |Q|$ for all $i \in \mathbb{N}$ and so that we
    obtain the second part of the claim. Then, since in every episode we have a
    non-zero probability of visiting a minimal even priority state,
    we obtain the first part of the claim as a
    simple consequence of the second Borel-Cantelli lemma.
\end{proof}

\newcommand{\fcstrat}{\tau_{\mathrm{fin}}}

We now turn to learning using finite memory only. Consider parameters
$\epsilon,\gamma \in (0,1)$. Let $\eta = \min\{\pmin,\eta_{\epsilon/4}\}$ for
$\eta_{\epsilon/4}$ as defined for Lemma~\ref{lem:robust-opt}. The strategy
$\fcstrat$ that we construct does the following.
\begin{enumerate}
    \item First, it computes $\delta'$ such that $\delta' \eclose{\eta}{}
        \delta$ with probability at least $1 - \gamma$ and a reward function
        $r'$ by following an exploration strategy $\lambda$ for $\alpha_T$
        during $L$ steps (see Lemma~\ref{lem:suff-urandom}).
    \item Then, it computes a unichain deterministic expectation-optimal
        strategy $\sigma_{\mathrm{MP}}^{\delta'}$ for
        $\autotomdp{\calA}{\delta'}{r'}$ and repeats the following forever:
        follow $\sigma_{\mathrm{MP}}^{\delta'}$ for $O$
        steps, then follow $\lambda$ for $|Q|$ steps.
\end{enumerate}
Using the fact that, in a finite MC with a single BSCC,
almost all runs obtain the expected mean payoff
and the assumption that the EC is good, one can
then prove the following result.
\begin{proposition}\label{pro:good-ec-fin}
    For all $\epsilon,\gamma \in (0,1)$
    one can compute $L,O \in \mathbb{N}$ such that for the resulting strategy
    $\fcstrat$, for all $q_0 \in Q$, for
    all $\delta$ compatible with $\calA$ and $\pmin$, and for all reward
    functions $r$, we have
    % parity
        $\probevent{\autotochain{\calA}{\delta}{r}{\fcstrat}}{q_0}{\parity}
            = 1$ and
    % mean payoff
        $\probevent{\autotochain{\calA}{\delta}{r}{\fcstrat}}{q_0}{
                \rho : \MP(\rho) \ge \Val(Q,\alpha_T) - \epsilon
            } \ge 1 - \gamma$.
\end{proposition}

\subparagraph{Optimality.}
Obviously, the result of Proposition~\ref{pro:inf-mem-better} is optimal as we
obtain the best possible value with probability one. We claim that the result of
Proposition~\ref{pro:good-ec-fin} is also optimal as we have seen that when we
use finite learning, we cannot do better than $\epsilon$-optimality with high
probability even in if no parity constraint is imposed (see
Proposition~\ref{pro:opt-fin-mem}).

\subsection{The case of a single end component}
Consider an almost-surely-good automaton $\calA = (Q,A,T,p)$ such that
$(Q,\alpha_T)$ is an EC and some $\pmin \in (0,1]$. The EC
is not necessarily good but as the automaton is almost-surely-good, then we have
the analogue of Lemma~\ref{lem:contain-good-ec} in this context.
\begin{lemma}\label{lem:contain-good-ec2}
    For all almost-surely-good automata $\calA = (Q,A,T,p)$ such that
    $(Q,\alpha_T)$ is an EC there exists $(S,\beta) \subseteq (Q,\alpha_T)$ such
    that $(S,\beta)$ is a GEC in $\autotomdp{\calA}{\delta}{r}$ for all $\delta$
    compatible with $\calA$ and all reward functions $r$, i.e. $(Q,\alpha_T)$ is
    weakly good.
\end{lemma}

\subparagraph{Yardstick.} 
As a yardstick for this case, we use the following value:
\(
    \lcVal(Q,\alpha_T) \defequals
    %\max\{\Val(S,\beta) \st (S,\beta) \subseteq (Q,\alpha_T) \text{ and }
    %(S,\beta) \text{ is a GEC} \}.
    \max_{q \in Q} \sup\left\{
    \expect{\calA^\sigma_{\delta,r}}{q}{\MP}
    \:\middle|\:
    \sigma \text{ is an almost-surely parity-winning strategy}
    \right\}.
\)
That is, $\lcVal(Q,\alpha_T)$ is the best expected mean-payoff value that can be
obtained while satisfying the parity objective almost surely.
%in a GEC included in the EC.
%Such a good EC exists by
%Lemma~\ref{lem:contain-good-ec2}.

\subparagraph{Learning strategy.} We will now prove an analogue of
Proposition~\ref{pro:fbstrat2}. For any given $\epsilon,\gamma \in (0,1)$ we
define the strategy $\sigma$ as follows.
\begin{enumerate}
    \item First, it follows an exploration strategy $\lambda$ for $\alpha_T$
        during sufficiently many steps (say $K$) to compute an approximation
        $\delta'$ of $\delta$ such that $\delta' \eclose{\eta_{\epsilon/4}}{}
        \delta$ with probability at least $1 - \gamma/2$; and a reward function
        $r'$ (see Lemma~\ref{lem:suff-urandom}).
    \item Next, it selects a GEC $(S,\beta)$ with maximal value
        $\pm\frac{\epsilon}{4}$ (see Lemma~\ref{lem:contain-good-ec2}) and
        computes for it a strategy $\tau$ as in
        Proposition~\ref{pro:good-ec-fin} with
        $\epsilon_1/2$ and $\gamma/2$.
    \item Finally, $\sigma$ follows $\lambda$ until
        $(S,\beta)$ is reached, then switches to $\tau$.
\end{enumerate}
%
%It is straightforward to prove the following about the constructed strategy.
\begin{proposition}\label{pro:fmstrat}
    For all $\epsilon,\gamma \in (0,1)$
    one can construct a finite-memory strategy
    $\sigma$ such that for all $q_0 \in Q$,
    for all $\delta$ compatible with $\calA$ and
    $\pmin$, and for all reward functions $r$, we have
    % parity
        $\probevent{\autotochain{\calA}{\delta}{r}{\sigma}}{q_0}{\parity}
            = 1$ and
    % mean payoff
        $\probevent{\autotochain{\calA}{\delta}{r}{\sigma}}{q_0}{
                \rho : \MP(\rho) \geq \lcVal(Q,\alpha_T) - \epsilon
            } \geq 1 - \gamma$.
\end{proposition}

\subparagraph{Optimality.}
Using the same example and reasoning as in Proposition~\ref{prop:opt-sure-ec},
we can show that this result is optimal and cannot be improved. Also note that
using infinite memory would not help as shown with the example of
Fig.~\ref{fig:opt3}, where the learning needs to be finite and enforcing the
almost sure parity does not require infinite memory.

\subsection{General almost-surely-good automata}
We now generalize our approach to general
almost-surely-good automata.
\begin{theorem}\label{thm:case2}
    Consider an almost-surely-good automaton $\calA = (Q,A,T,p)$
    and some $\pmin \in (0,1]$. For all
    $\epsilon,\gamma \in (0,1)$ one can compute a finite-memory
    strategy $\sigma$ such that for all $q_0 \in Q$,
    for all $\delta$ compatible with
    $\calA$ and $\pmin$, and all reward functions $r$, we have
    \begin{itemize}
        \item $\probevent{\autotochain{\calA}{\delta}{r}{\sigma}}{q_0}{\parity}
            = 1$ and
        \item
            \(
                \probevent{\autotochain{\calA}{\delta}{r}{\sigma}}{q_0}{
                        \rho
                        :
                        \MP(\rho) \ge \lcVal(S,\beta) - \epsilon
                    \:\middle|\:
                    \Inf \subseteq S
                } \geq 1 - \gamma
            \)            
            for all $(S,\beta) \in \mec{\autotomdp{\calA}{\delta}{r}}$
            such that $(S,\beta)$ is weakly good and 
            $\probevent{\autotochain{\calA}{\delta}{r}{\sigma}}{q_0}{\Inf
            \subseteq S} > 0$.
    \end{itemize}
\end{theorem}
\begin{proof}
    The argument to prove the above result is simple:
    $\sigma$ follows a strategy $\asparstrat$ that ensures satisfying the
    parity objective almost surely. Then, if the run reaches a state contained in a
    weakly good MEC, $\sigma$ switches to $\tau$ as described in
    Proposition~\ref{pro:fmstrat}. Clearly, the strategy almost-surely satisfies
    the parity objective. Furthermore, by following $\tau$ if the weakly good
    MEC is reached, the mean-payoff part of the claim is implied by
    Proposition~\ref{pro:fmstrat}.
\end{proof}
See the remark in Sect.~\ref{sec:mp-fm} for a comment on the
finite-memory implementability of $\sigma$; the
remark in Sect.~\ref{sec:sure-gen-auto} for a word on how
to modify $\sigma$ to favour some unknown MECs.

%% file: conclu.tex
%conclu.tex
As future work, we would like to study different configurations resulting from
relaxations of the assumptions we make in this work (i.e. full support, $\pmin$, and bounded reward). 
Further, we would like to
obtain model-free learning algorithms ensuring the same guarantees we give here.
Finally, we have left open the choice of strategy driving the visits to MECs in
Theorems~\ref{thm:case1} and~\ref{thm:case2} (as long as it satisfies the parity
objective). Indeed, the question of computing an ``optimal'' such strategy in
view of the unknown components of the MDP can be addressed in different ways.
One such way would be to model the problem as a Canadian traveler
problem~\cite{py91}.

%% file: app.tex
\section{Proof of Lemma~\ref{lem:robust-opt}}
One of the results we cite, i.e.~\cite[Theorem 5]{chatterjee12}, comes from a
work of Chatterjee that focuses on stochastic parity games with the same
support. (In their nomenclature, they are structurally equivalent.) There, they
derive robustness bounds for MDPs with the discounted-sum function and use them
to obtain robustness bounds for MDPs with a parity objective.

We are, however, extending those results to MDPs with the mean-payoff function
(cf.~\cite{dhkp17}) making use of an observation by Solan~\cite{solan03}:
robustness bounds for discounted-sum MDPs extend directly to mean-payoff MDPs if
they do not depend on the discount factor. As can be observed in the cited
result, this is indeed the case. Furthermore, we need not adapt the bound in our
context since we are assuming that the reward function in our MDPs with
mean-payoff function assigns transitions rewards between $0$ and $1$.

\begin{proof}[Proof of Lemma~\ref{lem:robust-opt}]
    The bounds for $\eta_\epsilon$ and $\max\{|r(q,a,q') - r'(q,a,q')| :
    (q,a,q') \in \supp{\delta}\}$ are obtained directly from Solan's
    inequality~\cite[Theorem 6]{solan03}:
    \[
        \left|
        \sup_{\tau_1} \expect{\calM^{\tau_1}}{q_0}{\MP}
        -
        \sup_{\tau_2} \expect{\calN^{\tau_2}}{q_0}{\MP}
        \right|
        \leq
        \frac{4|Q|d(\delta,\delta')}{1 - 2|Q|d(\delta,\delta')} +
        \lVert r-r' \rVert_{\infty}
    \]
    where $\lVert r-r' \rVert_\infty \defequals \max\{|r(q,a,q')-r'(q,a,q')| :
    (q,a,q') \in \supp{\delta}\}$ and $\calN = (Q,A,\alpha,\delta',p,r')$.
    We can then use Chatterjee's observation that Solan's $d(\cdot,\cdot)$
    function is upper-bounded by $\frac{\eta_\epsilon}{\pmin}$ where
    $\eta_\epsilon$ here stands for the maximal absolute difference of
    transition probabilities from $\delta$ and $\delta'$~\cite[Proposition
    1]{chatterjee12}. One thus obtains the following reformulation of the above
    inequality.
    \begin{equation}\label{eqn:solan}
        \left|
        \sup_{\tau_1} \expect{\calM^{\tau_1}}{q_0}{\MP}
        -
        \sup_{\tau_2} \expect{\calN^{\tau_2}}{q_0}{\MP}
        \right|
        \leq
        \frac{4|Q|(\eta_\epsilon/\pmin)}{1 - 2|Q|(\eta_\epsilon/\pmin)} +
        \lVert {r-r'} \rVert_{\infty}
    \end{equation}
    It follows from the inequalities required of $\eta_\epsilon$ and
    $\lVert {r-r'} \rVert_\infty$ in Lemma~\ref{lem:robust-opt}, together with
    the above
    inequalities, that the left-hand side of Inequality~\eqref{eqn:solan} is at
    most $\epsilon/2$. Indeed, we have required them to be twice as small as
    necessary. This is because Chatterjee has shown (in the proof
    of~\cite[Theorem 5]{chatterjee12}) that if the optimal expected values of
    two structurally-equivalent MDPs differ by at most $\epsilon$, then a
    memoryless expectation-optimal strategy for one is
    $2\epsilon$-expectation-optimal for the other. The result thus follows.
\end{proof}

\section{Proof of Lemma~\ref{lem:suff-urandom}}
We recall Hoeffding's concentration bound for the binomial distribution.
\begin{proposition}[Hoeffding's inequalities]
    Let $X_1,\dots,X_n$ be independent random variables with domain bounded by
    the interval $[0,1]$ and let $M \defequals \frac{1}{n} \sum_{i=1}^n X_i$.
    For all $0 < \epsilon < 1$ the following hold.
    \begin{itemize}
        \item $\rvprob{\rvexpect{M} - M \ge \epsilon} \leq
            \exp(-2n\epsilon^2)$
        \item $\rvprob{M - \rvexpect{M} \ge \epsilon} \leq
            \exp(-2n\epsilon^2)$
        \item $\rvprob{|M - \rvexpect{M}| \ge \epsilon} \leq
            2\exp(-2n\epsilon^2)$
    \end{itemize}
\end{proposition}

We learn an approximation $\delta'$ of the transition function $\delta$ by
following a strategy in the MDP and remembering the number of ``experiments'' we
have conducted for each pair $(q,a) \in Q \times A$. We then stop sampling when
a sufficient number of experiments has been carried out.  To obtain an
approximation that matches a desired confidence interval, we need a bound on how
many experiments have to be carried out.

\begin{lemma}\label{lem:suff-samples}
    Consider an MDP $\calM = (Q,A,\alpha,\delta,p,r)$, an end component
    $(S,\beta)$ in $\calM$, and $\epsilon, \gamma \in (0,1)$. If we let
    $X^{q'}_{q,a}$ denote a Bernoulli random variable with success probability
    $\delta(q'|q,a)$, then 
    \begin{equation}\label{eq:suff-samples}
        k \ge \frac{\ln(2|Q|^2|A|) - \ln(\gamma)}{2\epsilon^2}
    \end{equation}
    samples of each $X^{q'}_{q,a}$, for all $q,q' \in Q$ and all $a \in
    \beta(q)$, suffice to be able to compute a transition function
    $\delta'$ such that 
    \[
        \rvprob{\delta' \eclose{\epsilon}{(S,\beta)} \delta} \geq 1- \gamma.
    \]
\end{lemma}
\begin{proof}
    Let us denote by $d^{q'}_{q,a}$ the empirical mean of the $k$ samples
    of $X^{q'}_{q,a}$.
    It follows from Hoeffding's two-sided inequality and from our
    choice of $k$ that
    \[
        \rvprob{\left|d^{q'}_{q,a} - \delta(q'|q,a)\right| \ge \epsilon} \leq
        \frac{\gamma}{|Q|^2|A|}
    \]
    for all $q,q' \in Q$ and all $a \in \beta(q)$. Hence, the
    probability that for some $q,q' \in Q$ and some $a \in \beta(q)$ we
    get $\left|d^{q'}_{q,a} - \delta(q'|q,a)\right| \geq \epsilon$
    is at most $\gamma$.
\end{proof}

We can now prove the result.
\begin{proof}[Proof of Lemma~\ref{lem:suff-urandom}]
    Let us start by recalling a tail bound for binomial distributions that
    follows from Hoeffding's inequalities. Let $X_1,\dots,X_m$ be independent
    Bernoulli random variables with success probability $\mu$. We want an upper
    bound on the probability that the random variable $Y \defequals \sum_{i=1}^m
    X_i$ (with binomial distribution) is less than some desired threshold $k \in
    \mathbb{N}$. If $k \leq m\mu$ then 
    \begin{align*}
        & \rvprob{Y \leq k - 1}\\
        =& \rvprob{-Y \geq -k + 1}\\
        =& \rvprob{m\mu - Y \ge m\mu -k + 1}\\
        =& \rvprob{\mu - \frac{1}{m}\sum_{i=1}^m X_i \geq \frac{m\mu -k +1}{m}}
    \end{align*}
    and $0 < (m\mu -k + 1)/m < 1$.
    From one of Hoeffding's one-sided inequalities we then obtain that
    \begin{equation}\label{eq:bin-tail}
        \rvprob{Y \leq k - 1} \leq \exp\left(-2\frac{(m\mu-k+1)^2}{m}\right).
    \end{equation}

    Let us consider an arbitrary state $q' \in e$. Observe that, from any state
    $q_0 \in e$, the measure of runs $q_0 a_0 \dots q_{|Q|}$ that start from
    $q_0$ and contain the infix $q'a$, i.e. $q_i a_i = q'a$ for some $0 \leq i
    \leq |Q|$, while following a uniform random exploration strategy
    $\lambda$ for $\beta$ during $|Q|$ steps is at least
    \[
        \mu \defequals \left(\frac{\pmin}{|A|}\right)^{|Q|}.
    \]
    At this point we would like to fix the value of $k$ (used in our discussion
    above) by using Lemma~\ref{lem:suff-samples} with
    $\epsilon$ and $\gamma/2$. Therefore, we will henceforth have
    \[
        k \defequals \left\lceil \frac{\ln(4|Q|^2|A|) -
        \ln(\gamma)}{2\epsilon^2} \right\rceil.
    \]
    
    Consider two states $q,q' \in Q$ and an action $a \in A$.  We want to ensure
    that $X^{q'}_{q,a}$, as defined in Lemma~\ref{lem:suff-samples}, is sampled
    at least $k$ times with high probability. For this purpose, we let $W_{q,a}$
    be a Bernoulli random variable with success probability $\mu$. Intuitively,
    $W_{q,a} = 1$ indicates that we have reached $q$ and from it played $a$
    while following $\lambda$ during $|Q|$ steps. We will use the bound given in
    Equation~\eqref{eq:bin-tail} to obtain a lower bound on the number $n$ of
    times the strategy $\lambda$ has to be followed for $|Q|$ steps. That is,
    since $\exp\left(-2{(n\mu-k+1)^2}/{n}\right)$ is eventually decreasing, we
    can compute $n$ large enough so that $n \geq k/\mu$ and
    \[
        \rvprob{\sum_{i=1}^n W_{q,a} \leq k - 1} \leq
        \frac{\gamma}{2|Q||A|}.
    \]
    Observe that $n$ will be polynomial in $\mu$ and $1/\mu$ (thus exponential
    in $|Q|$ and polynomial in $|A|$ and $1/\pmin$), and also in $k$ (and thus
    polynomial in $1/\epsilon$ and $\ln(1/\gamma)$) since it suffices to take
    $n$ larger than the maximum among $k/\mu$ and the second root of
    \[
        (n\mu -k + 1)^2 - \frac{n}{2}\left(\ln(2|Q||A|) - \ln(\gamma)
        \right) \ge 0.
    \]
    Hence, the probability that, after following $\lambda$ for $n$ episodes of
    $|Q|$ steps in $(S,\beta)$, some $X^{q'}_{q,a}$ has not yet been sampled
    sufficiently many times is at most $\gamma/2$. 
    
    From the above discussion it follows that after following $\lambda$ for $n$
    (not necessarily consecutive) episodes of $|Q|$ steps each, we have
    \[
        \rvprob{\delta' \not\eclose{\epsilon}{(S,\beta)} \delta \:\middle|\:
        \exists W_{q,a} : \sum_{i=1}^n W_{q,a} \leq k - 1 } \cdot
        \rvprob{\exists W_{q,a} : \sum_{i=1}^n W_{q,a} \leq k -1}
        \leq \frac{\gamma}{2}
    \]
    where $\delta'$ is as constructed in Lemma~\ref{lem:suff-samples}
    (i.e. from the empirical mean of the samples of the $X^{q'}_{q,a}$).
    To conclude, we observe that that we also have
    \[
        \rvprob{\delta' \not\eclose{\epsilon}{(S,\beta)} \delta \:\middle|\:
        \forall W_{q,a} : \sum_{i=1}^n W_{q,a} \geq k} \cdot
        \rvprob{\forall W_{q,a} : \sum_{i=1}^n W_{q,a} \geq k}
        \leq \frac{\gamma}{2}
    \]
    by Lemma~\ref{lem:suff-samples}. The fact that
    \[
        \rvprob{\delta' \not\eclose{\epsilon}{(S,\beta)} \delta} \leq \gamma
    \]
    then follows from the law of total probability.
\end{proof}

\section{On the convergence of the finite averages}
Let us fix an MDP $\calM = (Q,A,\alpha,\delta,p,r)$ for this section.

The following result is repeatedly used throughout the paper.
\begin{lemma}\label{lem:tracol}
    If $(Q,\alpha)$ is an EC then for all $q_0 \in Q$, for all
    unichain deterministic memoryless strategies $\mu$, we have
    \begin{romanenumerate}
        \item \label{itm:lim}
            \(
                \probevent{\calM^\mu}{q_0}{
                \rho : \MP(\rho)
                \ge \expect{\calM^\mu}{q_0}{\MP}
                } = 1; \text{ and}
            \)
        \item \label{itm:fin} for all $\epsilon \in (0,1)$, one can compute
            $M(\epsilon) \in \mathbb{N}$ (dependent only on $\pmin$, $|Q|$, and
            $|A|$) such that
            \(
                \probevent{\calM^\mu}{q_0}{
                   \rho : \forall k \ge M(\epsilon),\, \finMP(\rho(..k))
                    \ge \expect{\calM^\mu}{q_0}{\MP} - \epsilon
                } \geq 1 - \epsilon.
            \)
    \end{romanenumerate}
\end{lemma}

Before we prove the above lemma we recall a result by Tracol which is slightly
weaker.
\begin{proposition}[{\cite[Proposition 2]{tracol09}}]\label{pro:original-tracol}
    If $(Q,\alpha)$ is an EC then for all $q_0 \in Q$, for all
    unichain deterministic memoryless strategies $\mu$, 
    for all $\epsilon \in (0,1)$, one can compute $K_0 \in \mathbb{N}$ and
    $c_1,c_2 >0$ 
    (dependent only on $\pmin$, $|Q|$, and $|A|$) such that $\forall k \ge
    K_0$ we have
    \[
        \probevent{\calM^\mu}{q_0}{
           \rho : \finMP(\rho(..k))
           \ge \expect{\calM^\mu}{q_0}{\finMP_k} - \epsilon
        } \geq 1 - c_1\cdot\exp(-k\cdot c_2 \cdot \epsilon^2)
    \]
    where $\finMP_k$ is the function such that $\rho \mapsto \finMP(\rho(..k))$.
\end{proposition}

\begin{remark}
We observe that Tracol's result depends on a bound for the mixing time of the
induced Markov chain. From results
in~\cite{lw95} it follows that one can compute such a bound even in unknown
chains.
\end{remark}

We will need one final ingredient before proving the advertised lemma: a
strengthening of Tracol's result in which the comparison inside the probability
operator is with the expected mean payoff and not the expectation of the finite
average.
\begin{proposition}\label{pro:old-tracol}
    If $(Q,\alpha)$ is an EC then for all $q_0 \in Q$, for all
    unichain deterministic memoryless strategies $\mu$, 
    for all $\epsilon \in (0,1)$, one can compute $K_0 \in \mathbb{N}$ and
    $c_1,c_2 >0$ 
    (dependent only on $\pmin$, $|Q|$, and $|A|$) such that $\forall k \ge
    K_0$ we have
    \begin{equation}\label{eqn:want-tracol}
        \probevent{\calM^\mu}{q_0}{
           \rho : \finMP(\rho(..k))
           \ge \expect{\calM^\mu}{q_0}{\MP} - \epsilon
        } \geq 1 - c_1\cdot\exp(-k\cdot c_2 \cdot \epsilon^2).
    \end{equation}
\end{proposition}
\begin{proof}
    Our first observation is that, by definition of limit, if we were to replace
    Equation~\eqref{eqn:want-tracol} by
    \[
        \probevent{\calM^\mu}{q_0}{
           \rho : \forall k \ge M(\epsilon),\, \finMP(\rho(..k))
           \ge \lim_{\ell \in \mathbb{N}_{>
           0}}\expect{\calM^\mu}{q_0}{\finMP_\ell} -
           \epsilon
        } \geq 1 - c_1\cdot\exp(-k\cdot c_2 \cdot \epsilon^2)
    \]
    the claim would then be implied by Proposition~\ref{pro:original-tracol}.
    Hence, it suffices to prove that
    \begin{equation}\label{eqn:limit-swap-exp}
        \expect{\calM^\mu}{q_0}{\MP} = \lim_{\ell \in \mathbb{N}_{>
        0}}\expect{\calM^\mu}{q_0}{\finMP_\ell}.
    \end{equation}
    Recall that the reward function is bounded, i.e. all rewards are in $[0,1]$.
    Then, from the ergodic theorem for bounded irreducible unichain reward
    Markov chains~\cite[Theorem 1.10.2]{norris98} we get that
    \[
        \probevent{\calM^\mu}{q_0}{
            \rho : \text{the limit }
            \lim_{\ell \in \mathbb{N}_{>0}} \finMP(\rho(..\ell))
            \text{ exists}
        } = 1.
    \]
    (Technically, the ergodic theorem applies only to strongly-connected MCs.
    However, it clearly extends to unichain MCs for the mean-payoff function
    since it is prefix-independent and almost all runs reach the unique BSCC
    with probability $1$.) Finally, we can now apply Lebesgue's dominated
    convergence theorem and conclude that Equation~\ref{eqn:limit-swap-exp} does
    indeed hold.
\end{proof}

We will now prove the lemma in two parts.
\begin{proof}[Proof of Lemma~\ref{lem:tracol}]
    \item \paragraph*{Item~\ref{itm:fin}}
    Remark that for all $\epsilon$, there exists $K_1 \ge K_0$ such that
    \[
        1 - c_1 \cdot \exp(-k \cdot c_2 \cdot \epsilon^2) \leq 1 - 2^k.
    \]
    Recall that (from~\cite[Proof of Lemma 12]{brr17}) we know that
    \(
        \lim_{i\to\infty}\prod_{j=i}^\infty(1-2^{-j}) = 1.
    \)
    The latter means that
    \[
        \prod_{j=K_2}^\infty(1-2^{-j}) \ge 1 - \epsilon
    \]
    for some $K_2 \ge K_1$.  Let us set $M(\epsilon)$ to be the minimal such
    $K_2$.

    Let us denote by $E_\ell$ the event
    \[
        \bigcap_{k = K_2}^\ell\{\rho : \finMP(\rho(..k))
            \ge \expect{\calM^\sigma}{q_0}{\MP} - \epsilon\}.
    \]
    It follows from Proposition~\ref{pro:old-tracol} that the probability
    measure of $E_\ell$ is at least $\prod_{k = K_2}^\ell (1 - 2^{-k})$.
    Furthermore, we have that $E_j \subseteq E_i$ for all $i \leq j$. Hence,
    we get (see~\cite[Page 756]{bk08}) that
    \[
        \probevent{\calM^\mu}{q_0}{
        \bigcap_{\ell \geq K_2} E_\ell
        } \ge 
        \prod_{j=K_2}^\infty(1-2^{-j}) \ge 1 - \epsilon
    \]
    which concludes the proof.

    \item \paragraph*{Item~\ref{itm:lim}}
    We will now make use of item~\ref{itm:fin} to prove item~\ref{itm:lim}.
    Consider a sequence $(\epsilon_i)_{i \in \mathbb{N}}$ such that $\epsilon_i
    = 2^i$. It should be clear that, if we write $E_i$ for the event
    \[
        \left\{
            \rho \st \exists K_0 \in \mathbb{N},\forall k \ge K_0,\,
            \finMP(\rho(..k)) \ge \expect{\calM^\sigma}{q_0}{\MP} - \epsilon_i
        \right\}
    \]
    we have that $E_k \subseteq E_j$ for all $j \leq k$. Furthermore, it follows
    from item~\ref{itm:lim} that $\probevent{\calM^\mu}{q_0}{E_i} \geq 1 - 2^i$
    for all $i \ge 0$.
    Hence, we can once more use
    the limit of the probabilities of the $E_i$ and conclude that
    \[
        \probevent{\calM^\mu}{q_0}{
        \bigcap_{i \in \mathbb{N}} E_i
        } = \lim_{i \in \mathbb{N}} 1 - \epsilon_i = 1
    \]
    which proves the claim since the event measured above corresponds to the set
    of runs whose mean payoff is at least the expected mean payoff
\end{proof}

\section{Proof of Proposition~\ref{pro:mp-opt-ec}}
Let $S_i$ denote the sum of all steps of all episodes $j < i$, i.e.
\(
    S_i \defequals \sum_{j = 0}^{i-1} L_i + O_i.
\)

In the following lemma, we state the guarantees that $\imstrat$ enforces
when the sequence $(L_i,O_i)_{i \in \mathbb{N}}$ of parameters is chosen
appropriately. We need to introduce some notation.
Let $\rho = q_0 a_0 \dots$ be a run and $k,\ell \in
\mathbb{N}$ such that $k \leq \ell$.  We denote by $\finMP(\rho(k..\ell))$ the
(finite) average of the $(k,\ell)$-infix of $\rho$, i.e.
\(
    \finMP(\rho(k..\ell)) \defequals
    \frac{1}{\ell-k}\sum_{i=k}^{\ell-1} w(q_i,a_i,q_{i+1}).
\)
We write $\rho(..\ell)$ instead of $\rho(0..\ell)$. 
\begin{lemma}\label{lem:local-opt}
    For all sequences $(\epsilon_i)_{i \in \mathbb{N}}$ such that $0 <
    \epsilon_k < \epsilon_j$ for all $j < k$, one can compute $(L_i,O_i)_{i \in
    \mathbb{N}}$ such that $L_i \geq |Q|$ for
    all $i \in \mathbb{N}$; additionally, for all $q_0 \in Q$, 
    for all $\delta$ compatible with $\calA$ and $\pmin$, and
    for all reward functions $r$,
    we have
    \[
        \forall i \ge 1,\,
        \probevent{\autotochain{\calA}{\delta}{r}{\imstrat}}{q_0}{
            \rho : \forall k \in
            (S_i,S_{i+1}], \, \finMP(\rho(..k)) \ge
            \Val(Q,\alpha_T) -
            \epsilon_i
        } \ge 1 -
        \epsilon_i.
    \]
\end{lemma}
\begin{proof}
    Let $(L_i)_{i \in \mathbb{N}}$ be such that with probability at least $1 -
    \epsilon_{i+1}/4$ we have that
    \begin{itemize}
        \item $\delta_i \eclose{\eta}{} \delta$ for $\eta$ smaller
            than $\pmin$ (see Lemma~\ref{lem:suff-urandom}) and
        \item smaller than $\eta_\epsilon$
            so that $\sigma^{\delta_i}_{\mathrm{MP}}$ is
            $(\epsilon_{i+2}/4)$-robust-optimal with respect to the expected mean
            payoff (see Lemma~\ref{lem:robust-opt})
    \end{itemize}
    for all $i \in \mathbb{N}$. Observe that to approximate $\delta$ we need to
    follow $\lambda$ for episodes of $|Q|$ steps. The $L_i$ can thus be assumed
    to be multiples of $|Q|$.

    For the optimization part of each episode, 
    we set $O_i = M(\epsilon_{i+2}/4) + \max\{0,P_i,F_i\}$
    where $P_i$ and $F_i$ are inductively defined as follows
    \[
        P_i \defequals
        \left\lceil
        \left(S_i + L_i + M\left(\frac{\epsilon_{i+2}}{4}\right)\right)
        \left(\frac{2(R_i-\epsilon_{i+2})}{\epsilon_{i+2}}\right)
        \right\rceil,
    \]
    with $R_i = \max\{r_i(t) \st t \in \supp{\delta_i}\}$, and
    \[
        F_i \defequals
        \left\lceil
        \left(S_i + L_i + M\left(\frac{\epsilon_{i+2}}{4}\right) +
        P_i + L_{i+1} + M\left(\frac{\epsilon_{i+3}}{4}\right)\right)
        \left(
        \frac{\Val(Q,\alpha_T)-\epsilon_{i+1}}{\epsilon_{i+1}-\epsilon_{i+2}}
        \right)
        \right\rceil
    \]
    for all $i \in \mathbb{N}$. It is easy to see that, since the $\epsilon_i$
    are decreasing, the $P_i$ and $F_i$ are eventually positive. Additionally,
    the $P_i$ depend only on the length of the history after $L_i$;
    the $F_i$ has the same dependencies plus $P_i$ and $M(\epsilon_{i+2}/4)$.
    The existence of such sequences of integers is therefore
    guaranteed.

    Consider an arbitrary $i \ge 1$. It follows from our choice of $L_{i-1}$
    that $\delta_{i-1} \eclose{\eta}{} \delta$, for $\eta < \pmin$, with
    probability at least $1 - \epsilon_i/4$. Hence, with the same probability,
    we also have that $\supp{\delta_{i-1}} = \supp{\delta}$ thus also that
    $r_{i-1}$ coincides with $r$ (since we have seen all positive-probability
    transitions and witnessed their rewards).
    Also from our choice of $\eta$, and with the same probability, we have that
    $\sigma^{\delta_{i-1}}_{\mathrm{MP}}$ is $(\epsilon_{i+1}/4)$-robust-optimal.
    If we write $M_i = L_i + M(\epsilon_{i+2}/4)$ then from the above
    arguments and Lemmas~\ref{lem:robust-opt} and~\ref{lem:tracol}
    item~\ref{itm:fin} we get
    that
    \begin{equation}\label{eqn:this-episode}
    \begin{aligned}
        &
        \probevent{\autotochain{\calA}{\delta}{r}{\imstrat}}{q_0}{
            \rho : \forall k \in
            (M_{i-1}, S_{i}], \, \finMP(\rho(M_{i-1}..k)) \ge
            \Val(Q,\alpha_T) -
            \frac{\epsilon_{i+1}}{2}
        }\\
        \ge &
        (1 - \epsilon_{i}/4)(1-\epsilon_{i+1}/4)\\
        \ge &
        1 - \frac{\epsilon_i + \epsilon_{i+1}}{4}\\
        \ge &
        1 - \frac{\epsilon_i}{2}, \text{ since } \epsilon_{i+1} < \epsilon_i.
    \end{aligned}
    \end{equation}
    Moreover, from our choice of $P_i$ we have that if $r_i$ coincides with $r$
    then
    \begin{align*}
        & P_{i-1} \ge 
        (S_{i-1} + L_{i-1}
        + M(\epsilon_{i+1}/4))(R_{i-1}-\epsilon_{i+1})(2/\epsilon_{i+1})\\
        \implies & P_{i-1} \ge
        (S_{i-1} + L_{i-1}
        + M(\epsilon_{i+1}/4))(\Val(Q,\alpha_T) - \epsilon_{i+1})
        (2/\epsilon_{i+1}),\\
        & \text{since } 0 \le \Val(Q,\alpha_T) < R_{i-1}\\
        \iff & P_{i-1}(\epsilon_{i+1}/2) \ge 
        (S_{i-1} + L_{i-1}
        + M(\epsilon_{i+1}/4))(\Val(Q,\alpha_T) - \epsilon_{i+1})\\
        \iff & P_{i-1}(\Val(Q,\alpha_T) - \epsilon_{i+1}/2) -
        P_{i-1}(\Val(Q,\alpha_T) - \epsilon_{i+1})\\
        & \ge 
        (S_{i-1} + L_{i-1}
        + M(\epsilon_{i+1}/4))(\Val(Q,\alpha_T) - \epsilon_{i+1})\\
        \iff & P_{i-1}(\Val(Q,\alpha_T) - \epsilon_{i+1}/2)\\
        & \ge 
        (S_{i-1} + L_{i-1} + M(\epsilon_{i+1}/4) + P_{i-1})(\Val(Q,\alpha_T) -
        \epsilon_{i+1})\\
        \iff & \frac{P_{i-1} (\Val(Q,\alpha_T) - \epsilon_{i+1}/2)}{
        S_{i-1} + L_{i-1} + M(\epsilon_{i+1}/4) + P_{i-1}
        } \geq \Val(Q,\alpha_T) - \epsilon_{i+1}.
    \end{align*}
    Hence, we get that if we write $N_i = M_i + P_i$ then
    \begin{equation}\label{eqn:p-for-past}
        \probevent{\autotochain{\calA}{\delta}{r}{\imstrat}}{q_0}{
            \rho : \forall k \in
            (N_{i-1}, S_{i}], \, \finMP(\rho(..k)) \ge
            \Val(Q,\alpha_T) -
            \epsilon_{i+1}
        }
        \ge
        1 - \frac{\epsilon_i}{2}.
    \end{equation}

    Note that Equation~\eqref{eqn:p-for-past} holds for all $i \ge 1$. It
    follows that the desired result holds for $k \in (N_i,S_{i+1}]$ since
    $\epsilon_{i+1}/2 < \epsilon_{i+1} < \epsilon_i$. Therefore, to
    conclude the proof, all that remains is to argue that $F_{i-1}$ is large
    enough so that the claim also holds for all $k \in (S_i,N_i]$ (with the
    desired probability).

    Observe that from our choice of $F_i$ we have that
    \begin{align*}
        & F_{i-1} \geq \frac{(S_{i-1} + L_{i-1} + M(\epsilon_{i+1}/4) +
           P_{i-1} + L_{i} + M(\epsilon_{i+2}/4))
           (\Val(Q,\alpha_T)-\epsilon_{i})}{\epsilon_{i}-\epsilon_{i+1}}\\
        \iff & F_{i-1}(\epsilon_i - \epsilon_{i+1})\\
        & \geq
        (S_{i-1} + L_{i-1} + M(\epsilon_{i+1}/4) +
           P_{i-1} + L_{i} + M(\epsilon_{i+2}/4))
           (\Val(Q,\alpha_T)-\epsilon_{i})\\
        \iff & F_{i-1}(\Val(Q,\alpha_T) - \epsilon_{i+1}) -
        F_{i-1}(\Val(Q,\alpha_T) - \epsilon_i)\\
        & \geq 
        (S_{i-1} + L_{i-1} + M(\epsilon_{i+1}/4) +
           P_{i-1} + L_{i} + M(\epsilon_{i+2}/4))
           (\Val(Q,\alpha_T)-\epsilon_{i})\\
        \iff & F_{i-1}(\Val(Q,\alpha_T) - \epsilon_{i+1})\\
        & \geq
        (S_{i-1} + L_{i-1} + M(\epsilon_{i+1}/4) +
        P_{i-1} + F_{i-1} + L_{i} + M(\epsilon_{i+2}/4))
           (\Val(Q,\alpha_T)-\epsilon_{i})\\
        \iff & \frac{F_{i-1}(\Val(Q,\alpha_T) - \epsilon_{i+1})}{
            S_{i-1} + L_{i-1} +
           M(\epsilon_{i+1}/4) + P_{i-1} + F_{i-1} + L_{i} +
           M(\epsilon_{i+2}/4)} \geq 
           (\Val(Q,\alpha_T)-\epsilon_{i}).
    \end{align*}
    The above inequality implies that
    \begin{equation}\label{eqn:f-for-future}
        \probevent{\autotochain{\calA}{\delta}{r}{\imstrat}}{q_0}{
            \rho : \forall k \in
            (S_i, M_i], \, \finMP(\rho(..k)) \ge
            \Val(Q,\alpha_T) -
            \epsilon_{i}
        }
        \ge
        1 - \frac{\epsilon_i}{2}
    \end{equation}
    since all rewards are assumed to be non-negative. Furthermore, 
    Equations~\eqref{eqn:this-episode} and~\eqref{eqn:f-for-future}
    allow us to conclude that
    \[
        \probevent{\autotochain{\calA}{\delta}{r}{\imstrat}}{q_0}{
            \rho : \forall k \in
            (S_i, S_{i+1}], \, \finMP(\rho(..k)) \ge
            \Val(Q,\alpha_T) -
            \epsilon_{i}
        }
        \ge
        (1 - \epsilon_i/2)(1- \epsilon_{i+1}/2).
    \]
    The proof is thus complete since
    $(1 - \frac{\epsilon_i}{2})(1- \frac{\epsilon_{i+1}}{2}) \geq (1 -
    \epsilon_i)$ because $\epsilon_{i+1} < \epsilon_i$.
\end{proof}

We are now ready to prove the proposition making use of the above lemma.
\begin{proof}[Proof of Proposition~\ref{pro:mp-opt-ec}]
    Let $\epsilon_i$ be $2^{-i}$ for all $i \in
    \mathbb{N}$. Clearly, we have that $0 < \epsilon_k < \epsilon_j < 1$ for all
    $j < k$. 

    It follows from 
    Lemma~\ref{lem:local-opt} that we can compute
    $(L_i)_{i \in \mathbb{N}}$ and $(O_i)_{i \in \mathbb{N}}$ such that
    \[
        \forall i \ge 1, \,
        \probevent{\autotochain{\calA}{\delta}{r}{\sigma}}{q_0}{
            \rho : \forall k \in (S_i,S_{i+1}], \,
            \MP(\rho(..k)) \ge \Val(Q,\alpha_T) -
            \epsilon_i
        }
        \ge 1- \epsilon_i.
    \]
    Henceforth, we will refer to the event in the above equation as $E_{i}$ and
    to its complement as $\overline{E_i}$. Observe that
    $D \defequals \bigcup_{i \in \mathbb{N}} \bigcap_{j \ge i} E_j$ consists
    only of runs whose mean payoff is at least $\Val(Q,\alpha_T)$. Hence, to
    conclude, it suffices to show that the complement $\overline{D}$ of $D$ has
    probability $0$. Since $\sum_{i \in \mathbb{N}}
    \probevent{\autotochain{\calA}{\delta}{r}{\sigma}}{q_0}{\overline{E_j}}
    < \infty$, then the Borel-Cantelli lemma gives us that
    \[
        \probevent{\autotochain{\calA}{\delta}{r}{\sigma}}{q_0}{\bigcap_{i \in
            \mathbb{N}} \bigcup_{j \ge i} \overline{E_j}} = 0.
    \]
\end{proof}

\section{Proof of Proposition~\ref{pro:fbstrat}}
Let us define a strategy $\sigma$ which follows $\imstrat$ while keeping a
counter $k$ initially set to $K_0$ (whose value will depend on $\gamma$).
Intuitively, $k$ keeps track of how many times we have tried to reach a
state with minimal even priority since the last time it was reset.
At the end of episode $i$, the counter $k$ is incremented by $1$ if
\begin{equation}\label{eqn:when-stop}
    \left( 1 - \left(\pmin/|A|\right)^{|Q|} \right)^\ell \leq 2^{-k}
\end{equation}
where
\(
    \ell \defequals \left|\left\{ 0 \leq j \leq i \st L_j \geq |Q|
    \right\}\right|.
\)
Additionally, if no episode between $i$ and the last time $k$ was
incremented contains a visit to a state with the minimal even priority,
$\sigma$ switches to follow $\parstrat$ forever. Observe that if $L_i \ge
|Q|$ for infinitely many $i$, then the expression in the left part of
Inequality~\eqref{eqn:when-stop} decreases monotonically. (Hence, for runs
which never switch to $\parstrat$, $k$ will be increased infinitely often.)

\begin{proof}[Proof of Proposition~\ref{pro:fbstrat}]
    Let the $L_i$ and $O_i$ be chosen as in Proposition~\ref{pro:mp-opt-ec}.
    Further, let $S_i \defequals \sum_{j=0}^{i-1} L_i + O_i$ as in
    Section~\ref{sec:mp-opt-ec}. Additionally let the sequence $(J_i)_{i \in
    \mathbb{N}}$ be such that $J_0 = 0$ and $J_{i}=S_{\ell + 1}$, where
    $\ell$ is the minimal natural number such that,
    \(
        \left( 1 - \left(\frac{\pmin}{|A|}\right)^{|Q|}\right)^{\ell} \leq
        2^{-K_0 + i - 1}
    \)
    for all $i \ge 1$.

    It is not hard to see that for all $K_0 \in \mathbb{N}$ and all $m \ge K_0$,
    for all priorities $x$ such that $p(q) = x$ for some $q \in Q$, we have that
    \begin{equation}\label{eqn:prod-not-fall}
        \probevent{\autotochain{\calA}{\delta}{r}{\sigma}}{q_0}{
                q_0 q_1 \dots
                :
                \forall i \in [K_0, m],\,
                \exists j \in [J_i,J_{i+1}],\,
                p(q_j) = x
         } \ge \prod_{k = K_0}^m (1 - 2^{-k}).
    \end{equation}
    Indeed, this follows from the fact that every $L_i$ is at least $|Q|$ steps
    long and that $(Q,\alpha_T)$ is an end component. Let $E_m$ denote the event
    in the above equation.
    Observe that for $x$ the
    minimal even priority in the end component, the probability that $\sigma$
    does not switch to $\parstrat$ is at least
    \[
        \probevent{\autotochain{\calA}{\delta}{r}{\sigma}}{q_0}{
            \bigcap_{m \ge 0} E_m
         }.
    \]
    It then follows from the fact that
    \(
        \lim_{i\to\infty}\prod_{k=i}^\infty(1-2^{-k}) = 1
    \)
    \cite[Proof of Lemma 12]{brr17}
    that we can choose $K_0$ large enough so that 
    \(
        \prod_{k=K_0}^\infty(1-2^{-k}) \geq 1 - \gamma.
    \)
    Hence, since $E_j \subseteq E_i$ for all $i \leq j$,
    Equation~\eqref{eqn:prod-not-fall} gives us that the probability that
    $\sigma$ does not switch to $\parstrat$ is at least
    \[
        \probevent{\autotochain{\calA}{\delta}{r}{\sigma}}{q_0}{
            \bigcap_{m \ge 0} E_m
         } \ge 
        \prod_{k=K_0}^\infty(1-2^{-k}) \geq 1 - \gamma.
    \]
    This already implies 
    \(
        \probevent{\autotochain{\calA}{\delta}{r}{\sigma}}{q_0}{
                \rho 
                :
                \MP(\rho) \ge \Val(Q,\alpha_T)
        } = 1 - \gamma
    \)
    by Proposition~\ref{pro:mp-opt-ec}
    since $\sigma$ follows $\imstrat$ if it does not switch to $\parstrat$.

    To conclude, we argue that all runs consistent with $\sigma$ satisfy the
    parity objective. Indeed, if along a run, $\sigma$ starts following
    $\parstrat$, then the run satisfies the parity objective from then onward.
    Thus, it the whole run satisfies the objective by prefix-independence. If
    along a run, $\sigma$ does not switch to $\parstrat$, then the run can be
    cut into finite segments of increasing length which contain at least one
    visit to a state with the minimal even priority. (If this were not the case,
    there would have been a switch to $\parstrat$.) Hence, that priority is seen
    infinitely often and the parity objective is satisfied.
\end{proof}

\section{Proof of Lemma~\ref{lem:contain-good-ec}}
\begin{proof}[Proof of Lemma~\ref{lem:contain-good-ec}]
    Since $\calA$ is surely good, from all $q \in Q$ there is a parity-winning
    strategy. For simplicity, let $\parstrat$ be a uniform memoryless
    deterministic winning strategy implementable by a stochastic Mealy machine
    with a single memory element $m_0$.

    Consider an arbitrary $\delta$ compatible with $\calA$ and $\pmin$ and an
    arbitrary reward function $r$. Clearly, in
    $\autotochain{\calA}{\delta}{r}{\parstrat}$ there cannot be any cycles
    $\chi$ such that the minimal priority of a state in $\chi$ is
    odd---otherwise this would contradict the fact that $\parstrat$ is uniformly
    winning.

    Let $C \subseteq Q \times \{m_0\}$ be a strongly-connected component in
    $\autotochain{\calA}{\delta}{r}{\parstrat}$. From the above arguments we
    have that the minimal priority in $C$ must be even. It should then be clear
    that $(S,\beta)$, where $S = \{q \st (q,m_0) \in C\}$ and $\beta(q) =
    \supp{\parstrat(q)}$ for all $q \in S$, is a good end component in
    $\autotochain{\calA}{\delta}{r}{\parstrat}$. Since $\delta$ and $r$ were
    arbitrary, the result follows.
\end{proof}

\section{Proof of Proposition~\ref{pro:fbstrat2}}
\begin{proof}[Proof of Proposition~\ref{pro:fbstrat2}]
    For the sure satisfaction of the parity objective we just observe that every
    run is eventually consistent with a strategy $\tau$ obtained from
    Proposition~\ref{pro:fbstrat} or with a winning strategy. Hence, by
    prefix-independence of the parity objective, the claim holds.

    The approximated transition function $\delta'$ is such that $\delta'
    \eclose{\eta}{ } \delta$ with probability at least $1 - \gamma/4$.
    Hence, with the same probability, since $\eta \leq \pmin$, we have that
    $\supp{\delta'} = \supp{\delta}$ and that $r' = r$. By
    Lemma~\ref{lem:robust-opt}, and because of our choice of $\eta$,
    the value of our chosen GEC in
    $\autotomdp{\calA}{\delta'}{r'}$ is ``off'' by at most $\epsilon$. That
    is, if the chosen GEC is $(S,\beta)$ and $(S',\beta')$ is a
    GEC with maximal value then
    \[
        \left|
        \Val(S,\beta) - \Val(S',\beta')
        \right| \leq \epsilon.
    \]
    
    Now, let us consider the strategy $\tau$ obtained
    from Proposition~\ref{pro:fbstrat} with probability at least 
    $1 - \gamma/4$. The strategy
    guarantees that
    \[
        \probevent{\autotochain{\calA}{\delta}{r}{\tau}}{q_0}{
                \rho : \MP(\rho) \geq \Val(S,\beta)
        } \geq 1 - \gamma/4.
    \]
    From the above
    arguments, and the fact that we follow $\lambda$ until the probability that we
    have reached $(S,\beta)$ is at least $1 - \gamma/4$, it follows that
    \[
        \probevent{\autotochain{\calA}{\delta}{r}{\sigma}}{q_0}{
                \rho : \MP(\rho) \geq V - \epsilon
        } \geq (1 - \gamma/4)^3 \geq 1 - \gamma
    \]
    where $V = \max\{ \Val(S,\beta) \st (S,\beta) \subseteq (Q,\alpha_T) \text{
    is a GEC}\}$. Thus, to conclude, it suffices to argue that
    $\cVal(Q,\alpha_T) \leq V$.
    
    This fact had already been observed in~\cite{akv16} but we sketch a proof of
    it here for completeness. First, let us point out
    that the maximum over states in the definition of $\cVal(\cdot)$ is not
    needed because the value is the same for all states (as a consequence of it
    being an end component). Second, recall that the mean-payoff function is
    prefix-independent, and that almost all runs consistent with a strategy in
    an MDP are eventually trapped in an end component~\cite[Theorem
    10.120]{bk08}. It then follows that $\cVal(Q,\alpha_T)$ is bounded by a
    convex combination of the elements from $\{\Val(S,\beta) \st (S,\beta)
    \subseteq (Q,\alpha_T) \text{ is a GEC}\}$. There is no need
    to consider other ECs since $\cVal(Q,\alpha_T)$ is a supremum over all
    winning strategies only (hence, if runs consistent with them were trapped in
    bad end components, it would not be winning). The desired result then
    follows from properties of convex combinations.
\end{proof}

\section{Proof of Theorem~\ref{thm:case1}}
\begin{proof}[Proof of Theorem~\ref{thm:case1}]
    We now argue that $\sigma$ satisfies the parity objective surely. If, along
    a run, eventually $\parstrat$ is followed forever then by choice of
    $\parstrat$ and by prefix-independence of the parity objective, we have that
    the run satisfies it. Otherwise, the run eventually stays forever in a
    GEC, and it is following $\imstrat$. Then, by Proposition~\ref{pro:fbstrat}
    the run satisfies the parity objective from then onward. Hence, the run
    satisfies the parity objective (again by prefix-independence).

    Let us now focus on the mean payoff. Consider an arbitrary GEC
    $(S,\beta)$ with non-zero probability of being reached under $\sigma$. Since
    $\tau$ (from Proposition~\ref{pro:fbstrat2}) is followed the first time
    $(S,\beta)$ is entered, by Proposition~\ref{pro:fbstrat}, the definition of
    conditional probability, and prefix-independence of the mean payoff we get
    the desired result.
\end{proof}

\section{Proof of Proposition~\ref{pro:good-ec-fin}}
Let us first focus on a simplified version of the strategy $\fcstrat$. Namely,
let us suppose that we have access to $\sigma_{\mathrm{MP}}^{\delta'}$ a
$(\epsilon/4)$-robust-optimal strategy. The new strategy $\phi$ then plays in
episodes consisting of $|Q|$ steps during which $\lambda$ is followed, then $O$
steps during which $\sigma_{\mathrm{\MP}}^{\delta'}$ is followed.

It should be clear that we can obtain, in a similar way to how it is done for
Lemma~\ref{lem:local-opt} an $O$ large enough so that we have obtain
$(\epsilon/2)$-average-optimal episodes. That is, if we forget about all the
past and focus only on the steps contained in the current episode. Indeed, the
$O$ steps only need to account for the $|Q|$ sub-optimal steps carried out
previously in the same episode. We thus obtain the following result.
\begin{lemma}\label{lem:local-opt2}
    Given $\sigma_{\mathrm{MP}}^{\delta'}$ a $(\epsilon/4)$-robust-optimal
    strategy, one can compute $O \in \mathbb{N}$ such that for all $q_0 \in Q$,
    for all $\delta$ compatible with $\calA$ and $\pmin$, and for all reward
    functions $r$, we have
    \[
        \forall i \in \mathbb{N},\,
        \probevent{\autotochain{\calA}{\delta}{r}{\phi}}{q_0}{
            \rho : \forall k \in
            (S_i,S_{i+1}], \, \finMP(\rho(S_i..k)) \ge
            \Val(Q,\alpha_T) -
            \frac{\epsilon}{2}
        } \ge 1 - \frac{\epsilon}{2}.
    \]
\end{lemma}

We now proceed with the proof of the full claim.
\begin{proof}[Proof of Proposition~\ref{pro:good-ec-fin}]
    First, let us observe that $\fcstrat$ indeed ensures almost-sure
    satisfaction of the parity objective because it follows $\lambda$ for $|Q|$
    steps infinitely often. Thus, with non-zero probability we visit a state
    with minimal even priority in every episode. The second Borel-Cantelli lemma
    then gives us that the parity objective is satisfied almost surely.

    Let us assume that we have access to $\sigma_{\mathrm{MP}}^{\delta'}$ as for
    Lemma~\ref{lem:local-opt2}.  Now, using Lemma~\ref{lem:local-opt2} and the
    Bellman optimality equations for the limit of expected
    averages~\cite{puterman05,bk08} we obtain that
    \[
        \lim_{\ell \ge 1}
        \expect{\autotochain{\calA}{\delta}{r}{\phi}}{q_0}{\finMP_\ell} \geq
        \Val(Q,\alpha_T) - \epsilon.
    \]
    Observe now that $\phi$ is a unichain strategy. Hence, the equality we have
    established in Equation~\ref{eqn:limit-swap-exp} holds and we get that
    \[
        \expect{\autotochain{\calA}{\delta}{r}{\phi}}{q_0}{\MP} \geq
        \Val(Q,\alpha_T) - \epsilon.
    \]
    Furthermore, by Lemma~\ref{lem:tracol} item~\ref{itm:lim} the expectation is
    achieved with probability $1$.

    Now, all that remains is to show that can obtain
    $\sigma_{\mathrm{MP}}^{\delta'}$ a $(\epsilon/4)$-robust-optimal strategy
    with probability $1 - \gamma$.  However, this is a direct consequence of
    Lemma~\ref{lem:suff-urandom}, so the proof is complete.
\end{proof}

\section{Proof of Lemma~\ref{lem:contain-good-ec2}}
\begin{proof}[Proof of Lemma~\ref{lem:contain-good-ec2}]
    The proof of the claim goes almost identical to the argument used for the
    proof of Lemma~\ref{lem:contain-good-ec}.  Let $\asparstrat$ be a uniform
    memoryless deterministic almost-sure winning strategy implementable by a
    stochastic Mealy machine with a single memory element $m_0$. Consider
    arbitrary $\delta$ and $r$. In $\autotochain{\calA}{\delta}{r}{\asparstrat}$
    there cannot be any bottom strongly-connected components with a state whose
    minimal priority is odd.

    Let $C \subseteq Q \times \{m_0\}$ be any bottom strongly-connected
    component in $\autotochain{\calA}{\delta}{r}{\asparstrat}$. From the above
    arguments we have that the minimal priority in $C$ must be even. It should
    then be clear that $(S,\beta)$, where $S = \{q \st (q,m_0) \in C\}$ and
    $\beta(q) = \supp{\parstrat(q)}$ for all $q \in S$, is a good end component
    in $\autotochain{\calA}{\delta}{r}{\parstrat}$. Since $\delta$ and $r$ were
    arbitrary, the result follows.
\end{proof}

\section{Proof of Proposition~\ref{pro:fmstrat}}
\begin{proof}[Proof of Proposition~\ref{pro:fmstrat}]
    The proof closely follows the argument used to prove
    Proposition~\ref{pro:fbstrat2}.

    For the almost-sure satisfaction of the parity objective one only needs to
    observe that almost all runs are eventually consistent with
    a strategy $\tau$ obtained from
    Proposition~\ref{pro:good-ec-fin}. By
    prefix-independence of the parity objective, the claim thus holds.

    From our choice of $K$, 
    the approximated transition function $\delta'$ is such that $\delta'
    \eclose{\eta}{ } \delta$ with probability at least $1 - \gamma/2$.
    Hence, with the same probability, since $\eta \leq \pmin$, we have that
    $\supp{\delta'} = \supp{\delta}$ and that $r' = r$. By
    Lemma~\ref{lem:robust-opt}, and because of our choice of $\eta$, if the
    chosen GEC is $(S,\beta)$ and $(S',\beta')$ is a GEC with maximal value then
    \[
        \left|
        \Val(S,\beta) - \Val(S',\beta')
        \right| \leq \frac{\epsilon}{2}.
    \]
    
    Now, let us consider the strategy $\tau$ obtained
    from Proposition~\ref{pro:good-ec-fin}. Recall it
    guarantees that
    \[
        \probevent{\autotochain{\calA}{\delta}{r}{\tau}}{q_0}{
            \rho : \MP(\rho) \geq \Val(S,\beta) - \frac{\epsilon}{2}
        } \geq 1 - \gamma/2.
    \]

    From the above arguments, and the fact that we follow $\lambda$ until the
    chosen GEC $(S,\beta)$ is reached, we have that
    \[
        \probevent{\autotochain{\calA}{\delta}{r}{\sigma}}{q_0}{
                \rho : \MP(\rho) \geq V - \epsilon
        } \geq (1 - \gamma/2)^2 \geq 1 - \gamma.
    \]
    where $V = \max\{ \Val(S,\beta) \st (S,\beta) \subseteq (Q,\alpha_T) \text{
    is a GEC}\}$. Thus, to conclude, it suffices to argue that
    $\lcVal(Q,\alpha_T) \leq V$. However, in this context it is straightforward
    to show that $\lcVal(Q,\alpha_T)$ is in fact equivalent to $V$. This is
    because the MDP consists of a single end component, the mean-payoff function
    is prefix-independent, and almost all runs consistent with a strategy in an
    MDP are eventually trapped in an end component~\cite[Theorem 10.120]{bk08}.
    Hence, the result follows.
\end{proof}